\documentclass[journal]{IEEEtran}
\usepackage{amsmath,amsfonts}
\usepackage{amssymb}
\usepackage{amsthm}
\usepackage{algorithmic}
\usepackage{algorithm}
\usepackage{array}
\usepackage[caption=false,font=normalsize,labelfont=sf,textfont=sf]{subfig}
\usepackage{textcomp}
\usepackage{stfloats}
\usepackage{url}
\usepackage{verbatim}
\usepackage{graphicx}
\usepackage{cite}
\usepackage{multirow}

\usepackage{changes}
\definechangesauthor[name={Ping Luo}, color=green]{A1} 

\usepackage[colorlinks,linkcolor=black,anchorcolor=black,citecolor=black,urlcolor=black]{hyperref}
\newtheorem{theorem}{Theorem}

\newtheorem{assumption}{Assumption}

\usepackage{tikz,xcolor,hyperref}
\definecolor{lime}{HTML}{A6CE39}
\DeclareRobustCommand{\orcidicon}{%
	\begin{tikzpicture}
		\draw[lime, fill=lime] (0,0) 
		circle [radius=0.16] 
		node[white] {{\fontfamily{qag}\selectfont \tiny ID}};    \draw[white, fill=white] (-0.0625,0.095) 
		circle [radius=0.007];    \end{tikzpicture}
	\hspace{-2mm}}
\foreach \x in {A, ..., Z}{%
	\expandafter\xdef\csname orcid\x\endcsname{\noexpand\href{https://orcid.org/\csname orcidauthor\x\endcsname}{\noexpand\orcidicon}}
}

\usepackage{ulem}
\usepackage{cancel}

\begin{document}
	
	\title{FedVeca: Federated Vectorized Averaging on Non-IID Data with Adaptive Bi-directional Global Objective}
	\author{Ping Luo\orcidA{},~\IEEEmembership{Student Member,~IEEE,} Jieren Cheng,~\IEEEmembership{Member,~IEEE,} N. Xiong\orcidC{},~\IEEEmembership{Senior Memeber,~IEEE,} Zhenhao Liu ,~\IEEEmembership{Student Member,~IEEE,} and Jie Wu\orcidD{},~\IEEEmembership{Fellow,~IEEE}
	\thanks{This paper (excluding appendices) has been accepted for publication in the IEEE Transactions on Parallel and Distributed Systems, DOI:\href{https://ieeexplore.ieee.org/document/10664503}{10.1109/TPDS.2024.3454203}. \textit{(Corresponding author: Jieren Cheng.)}}
	\thanks{Ping Luo, Jieren Cheng, N. Xiong, and Zhenhao Liu are with the Hainan Blockchain Technology Engineering Research Center, Haikou, 570228, China (e-mail: \href{luoping@hainanu.edu.cn}{luoping@hainanu.edu.cn}; \href{cjr22@163.com}{cjr22@163.com};  \href{nnxiong10@yahoo.com}{nnxiong10@yahoo.com}
	\href{lzhcs11@163.com}{lzhcs11@163.com};).}
	\thanks{Jie Wu is with the Department of Computer and Information Sciences Temple University SERC 362 1925 N. 12th Street Philadelphia, PA 19122 (e-mail: \href{jiewu@temple.edu}{jiewu@temple.edu}).}
}
		
	\markboth{Journal of \LaTeX\ Class Files,~Vol.~14, No.~8, August~2021}%
	{Shell \MakeLowercase{\textit{et al.}}: A Sample Article Using IEEEtran.cls for IEEE Journals}
		
	\IEEEpubid{0000--0000/00\$00.00~\copyright~2021 IEEE}
		
	\maketitle

\begin{abstract}
Federated Learning (FL) is a distributed machine learning framework in parallel and distributed systems. However, the systems' Non-Independent and Identically Distributed (Non-IID) data negatively affect the communication efficiency, since clients with different datasets may cause significant gaps to the local gradients in each communication round. In this paper, we propose a Federated Vectorized Averaging (FedVeca) method to optimize the FL communication system on Non-IID data. Specifically, we set a novel objective for the global model which is related to the local gradients. The local gradient is defined as a bi-directional vector with step size and direction, where the step size is the number of local updates and the direction is divided into positive and negative according to our definition. In FedVeca, the direction is influenced by the step size, thus we average the bi-directional vectors to reduce the effect of different step sizes. Then, we theoretically analyze the relationship between the step sizes and the global objective, and obtain upper bounds on the step sizes per communication round. Based on the upper bounds, we design an algorithm for the server and the client to adaptively adjusts the step sizes that make the objective close to the optimum. Finally, we conduct experiments on different datasets, models and scenarios by building a prototype system, and the experimental results demonstrate the effectiveness and efficiency of the FedVeca method.
\end{abstract}


\section{Introduction}
\IEEEPARstart{P}arallel and distributed systems have transformed computing by allowing tasks to be executed simultaneously across multiple nodes, enhancing performance and scalability in high-performance computing environments, data centers, and cloud services \cite{gu2020towards}. With the rise of big data and artificial intelligence, these systems have become crucial for distributed machine learning frameworks that handle vast datasets across different locations while maintaining efficiency \cite{jiang2021parallel}. For real-world applications of distributed machine learning, considerations must be made for data silos resulting from privacy concerns, leading to the emergence of Federated Learning (FL) as an extension of parallel and distributed systems \cite{mcmahan2017communication}.

FL allows decentralized clients to train a global model using their own local data, thus alleviating the problems of data silos and user privacy \cite{konevcny2016federated, liu2022distributed}. FL local training generally adopts the optimization method of Stochastic Gradient Descent (SGD) \cite{wang2021cooperative}, which requires independent and identically distributed (IID) data to ensure unbiased estimation of global gradient in the training process \cite{dvurechensky2018computational, lei2019stochastic, harvey2019tight}. Since the clients' local data are obtained from the local environment and usage habits, the generated datasets are usually Non-IID due to differences in size and distribution \cite{aledhari2020federated,zhang2021federated2,gao2022feddc}. The Non-IID datasets and multiple local SGD iterations will bring a drift to the global model in each communication round, which significantly reduces the FL performance and stability in the training process, thus requiring more communication rounds to converge or even fail to converge \cite{sattler2019robust, luo2021no, zhang2021federated}. Therefore, it becomes a key challenge to optimize the FL objective by reducing the negative impact of Non-IID data.

\begin{figure}[!t]
	\centering
	\includegraphics[width=3.5in]{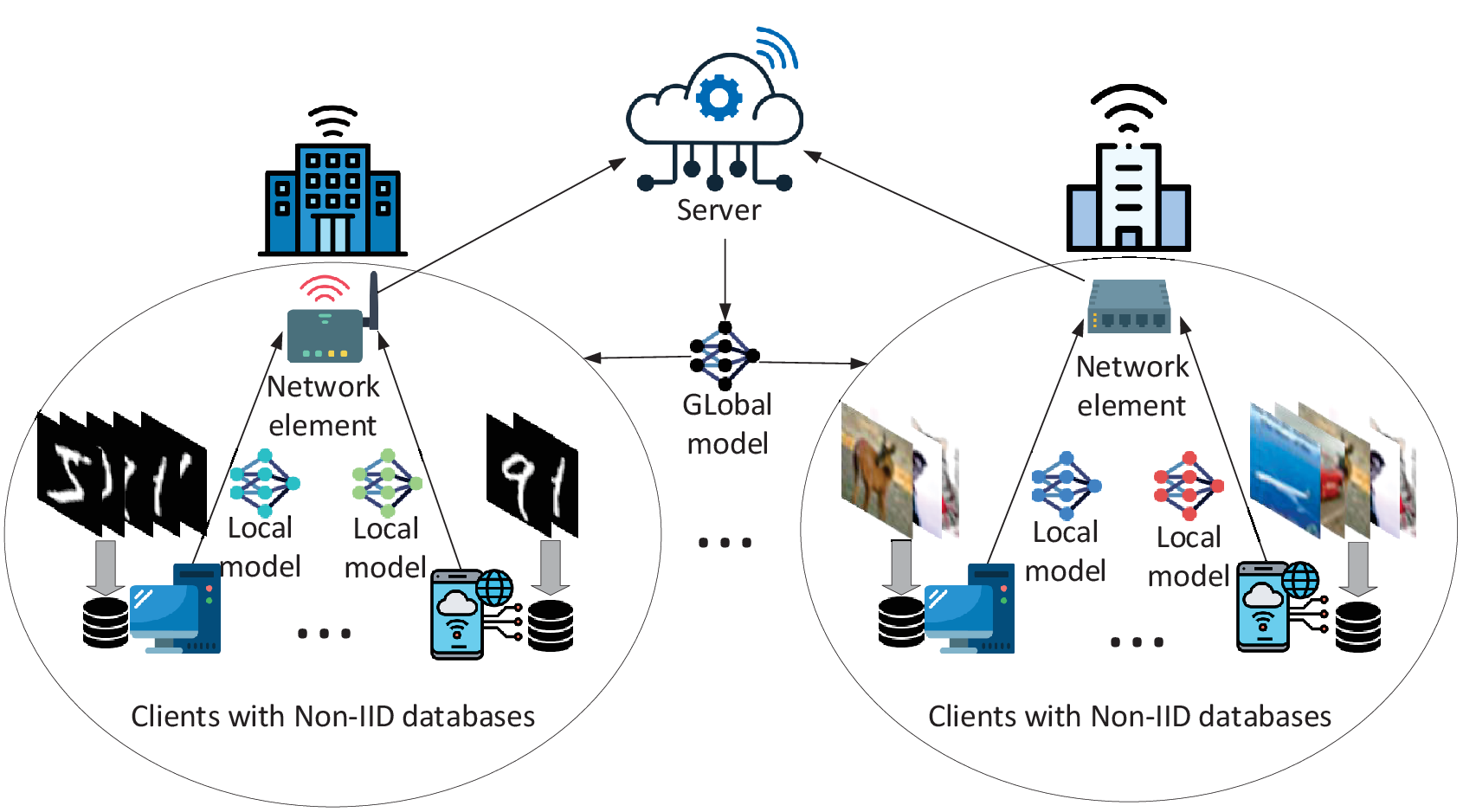}
	\caption{A typical FL communication networks system.}
	\label{system}
\end{figure}

\IEEEpubidadjcol

As illustrated in Fig. \ref{system}, we consider a typical Federated Averaging Algorithm (FedAvg) \cite{mcmahan2017communication} communication networks system, where clients are composed of communication-connected devices with local Non-IID training databases. Intuitively, the negative impact is due to the fact that the client has Non-IID data, and the common strategy is to adjust the data distribution from an optimization perspective of the data, thus effectively improving the generalization performance of the global model \cite{zhu2021federated}. Data distribution is usually a fundamental property of FL that is not easily changed, and a more easily tuned parameter is the number of local SGD iterations as a way of counteracting global objective drift \cite{wang2019adaptive}. Furthermore, our previous research has revealed that the number of local SGD iterations correlates with the Non-IID nature of datasets on individual clients \cite{cheng2022aafl}. Hence, we proceeded to investigate its viability within a more sophisticated FL framework, namely, Normalized Averaging Federated learning (FedNova) \cite{wang2020tackling}.

In this paper, we propose a Federated Vectorized Averaging (FedVeca) method to make the global model close to the optimal global model on Non-IID data. Specifically, we define the local gradients as a bi-directional vectors with positive and negative (two categories of directions). Positive vectors represent local gradients that exhibit less deviation from the global gradient vector, while negative vectors indicate a greater disparity. The visual representation of the positive and negative directions of the bi-directional vector is provided in Fig. \ref{problem}, with the analysis detailed in Section \ref{goal}. The standard quantization and classification of these vectors are detailed in Theorem \ref{theorem-2} of Section \ref{sec-analysis}. In addition to the direction, we define the step size for the bi-directional vector as the number of local SGD iterations. Then, we adaptively adjust the step size of the averaged bi-directional vector at each client to optimize global model. As shown in Fig. \ref{architecture}, the FedVeca method starts with the selection of the dataset and the construction of the model. Next, the global model is distributed from the server to the individual clients and then trained on the local dataset $ D_{1 \sim 3} $ to obtain the local gradients and control parameters. Local gradients are weighted and aggregated into the global gradient for the global model update, and the control parameters act on the averaged bi-directional vector to optimize the local step size $ \tau_{1 \sim 3} $ of the clients. Finally, the updated global model and the local step size are sent to the individual clients for the next round of updates.

\begin{figure*}[!t]
	\centering
	\includegraphics[width=6.0in]{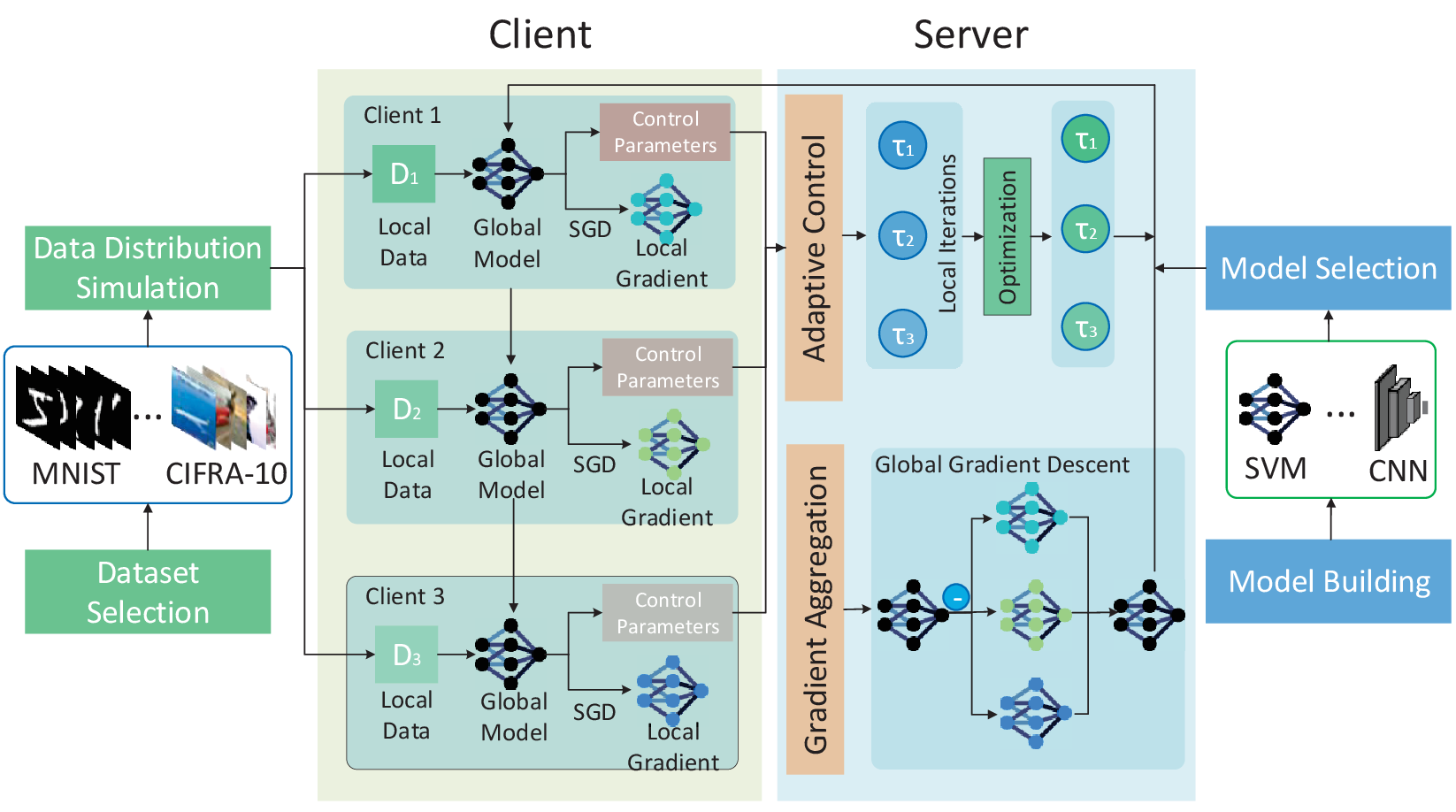}
	\caption{The overall architecture of proposed Federated Vectorized Averaging (FedVeca) method on Non-IID data.}
	\label{architecture}
\end{figure*}

The main contributions of this paper are as follows:
\begin{itemize}
	\item	We set a novel objective for the global model of FedVeca. In a qualitative analysis of the relationship between the global objective and the bi-directional vector's step size, we obtained an upper bound on the step size of each client in each communication round.
	\item	Using the above theoretical upper bound, we propose an algorithm for FedVeca that adaptively balances the step size at each client to accelerate the speed of the global model convergence.
	\item	We evaluate the general performance of the proposed algorithm via extensive experiments on general public datasets, which confirms that the FedVeca algorithm provides efficient performance in the Non-IID datasets.
\end{itemize}

This paper is organized as follows. Related work is introduced in Section \ref{sec-related work}. The basics and definitions of FL are summarized in Section \ref{sec-preliminaries}. The convergence analysis and control algorithm are presented in Sections \ref{sec-analysis}. Experimentation results are shown in Section \ref{sec-experiment} and the conclusion is presented in Section \ref{sec-conclusion}.

\section{Related Work}
\label{sec-related work}
Based on the strategy of adjusting the distribution of dataset on FL clients, there is a mean-augmented FL strategy \cite{yoon2020fedmix} which is inspired by data augmentation methods \cite{zhang2018mixup} and allows each client to average the data after exchanging updated model parameters. And the averaged data are sent to each client and used to reduce the degree of local data distribution imbalance. Moreover, other related studies treat each client as a domain, and data augmentation is applied to each client's data by selecting transformations related to the overall distribution to generate similar data distributions \cite{zhou2022domain, back2022mitigating}. Rather than modifying data directly, some schemes indirectly improve data distribution through client-side selection. For example, the FL process can be stabilized and sped up by describing the data distribution on the client through the uploaded model parameters, thus the best subsets of clients are intelligently selected in each communication round  \cite{wang2020optimizing}. Similarly, there is a scheme to evaluate the class distribution without knowing the original data, which uses a client selection strategy oriented towards minimizing class imbalance \cite{yang2021federated}. And the method proposed in \cite{zhao2018federated} creates a globally shared subset of data to obtain better learning performance in the case of Non-IID data distributions. Besides, there are some schemes that share part of local data to the central server for mitigating the negative impact of Non-IID data \cite{tuor2021overcoming, yoshida2019hybrid}.

In addition to the optimization of FL strategies by improving data distribution, the broader focus aims at adapting the model to local tasks. For adapting the model, some methods create a better initial model by local fine-tuning. The scheme proposed in \cite{fallah2020personalized} makes the local model have a better initial global model by using Model-Agnostic Meta-Learning (MAML) \cite{finn2017model}. As an extension, \cite{t2020personalized} proposes a federated meta-learning formulation using Moreau envelopes. Besides, Multi-tasking is proposed to solve statistical challenges in FL environments by finding relationships between clients' data such that similar clients learn similar models \cite{smith2017federated, corinzia2019variational, chen2018federated}. With the idea of extracting information from large models to small models, knowledge distillation has also been generalized to optimization strategies for FL \cite{zhu2021data}. For example, a distillation framework for model fusion with robustness is proposed in \cite{lin2020ensemble}, using unlabeled data output from client models to train a central classifier that flexibly aggregates heterogeneous models of clients. \cite{peng2019federated} proposed a domain-adaptive federated optimization method that aligns the learned representations of different clients with the data distribution of the target nodes and uses feature decomposition to enhance knowledge transfer.

The above algorithms are based on the same number of local SGD iterations per client in the training process, and we abbreviate this setup as synchronous FL in the paper. In synchronous FL optimization, the training time always depends on the slowest training node \cite{wang2020tackling}. Besides, some clients are unable to complete the training tasks within the specified time, and these clients will be discarded under normal circumstances, thus affecting the performance and accuracy of the trained model \cite{mcmahan2017communication}. Since the synchronous FL does not conform to the update principles of clients in real-world environments, it is necessary for each client to have a different number of local SGD iterations. But on Non-IID data, clients with different numbers of local SGD iterations lead to a problem where the global model is biased in the direction of the local model with a higher number of local SGD iterations. To solve this problem, Proximal Federated Learning (FedProx) \cite{li2020federated} introduces a regularization term in the local loss function to adjust the distance from the global model, so that there is no need to manually adjust the number of local SGD iterations. Moreover, \cite{liang2019variance} proposed a local SGD with reduced variance and can further reduce the communication complexity by removing the dependence of different clients on gradient variance, which has a significant performance on Non-IID datasets. Further, Stochastic Controlled Federated Learning (SCAFFOLD) \cite{karimireddy2020scaffold} corrects the direction of global model training by using a control variable (reduction variance), that solves the client-drift problem caused by local SGD iterations on Non-IID datasets. In the FedNova scheme \cite{wang2020tackling}, a regularization is introduced based on the contribution of the local gradients, limiting the bias problem for the global model caused by excessive local SGD iterations.

Building on the FedNova framework, our proposed FedVeca method introduces a strategy to adaptively regulate the local step size, which is contingent on the positive degree of the local gradient at each node. This strategy aims to modulate both the direction and magnitude of the local gradient vector. In the subsequent section, we formulate the training procedure and objectives of FedVeca, providing a proof of convergence along with the corresponding algorithm.

\section{Preliminaries and Definitions}
\label{sec-preliminaries}
In this section, we introduce some related preliminaries and our definitions. For convenience, we assume that all vectors are column vectors in this paper and use $ \boldsymbol{w}^{T} $ to denote the transpose of $ \boldsymbol{w} $. We use ``$ \triangleq $" to denote ``is defined to be equal to" and use $ \left\|.\right\| $ to denote the $ L $ norm.
 \subsection{Generalized Update Rules of FL}
 \label{update rule}
\begin{figure}[!t]
	\centering
	\includegraphics[width=3.5in]{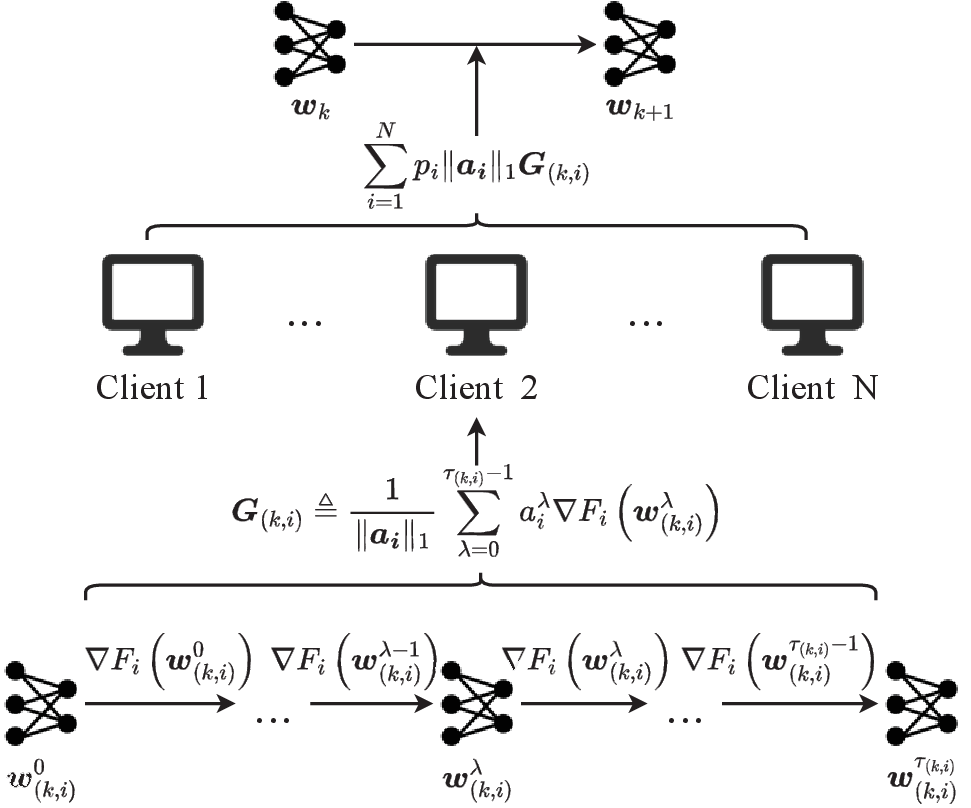}
	\caption{Generalized update rules in the $ k $-th round.}
	\label{fedavg}
\end{figure}
Assume that we have $ N $ clients with local train datasets $ \mathcal{D}_{1}, \mathcal{D}_{2}, ..., \mathcal{D}_{i}, ..., \mathcal{D}_{N} $, where $ i $ denotes the client index. FL performs the training process on these local train datasets under generalized update rules \cite{wang2020tackling}. In the generalized update rules, the training process has $ K $ ($ K \geq 1 $) communication rounds and the server has global model parameters $ \boldsymbol{w}_{k} $, where $ k = 0,1,2, ..., K $ denotes the round index. In $ k $ round, each client $ i $ has $ \tau_{(k,i)} $ ($ \tau_{(k,i)} \geq 1 $) local SGD iterations  and its local model parameters $ \boldsymbol{w}_{(k,i)}^{\lambda} $, where $ \lambda = 0,1,2,..., \tau_{(k,i)} $ denotes the local SGD iteration index. When FL begins ($ k=0 $), the server initializes the global model parameters $ \boldsymbol{w}_{0} $ and sends it to all clients. At $ \lambda=0 $, the local model parameters for all clients are received from the server that $ \boldsymbol{w}_{(k,i)}^{0} = \boldsymbol{w}_{k} $. For $ 0 \leq \lambda \leq \tau_{(k,i)} - 1 $, the gradients $ \nabla F_{i}\big(\boldsymbol{w}_{(k,i)}^{\lambda}\big) $ are computed according to the local loss function $ F_{i}\left(\boldsymbol{w}\right) $ and the local model parameters $ \boldsymbol{w}_{(k,i)}^{\lambda} $, and $ \boldsymbol{w}_{(k,i)}^{\lambda} $ are updated by local SGD update rule which is
\begin{equation}
	\label{local update}
	\boldsymbol{w}_{(k,i)}^{\lambda+1} - \boldsymbol{w}_{(k,i)}^{\lambda} = - \eta \nabla F_{i}\big(\boldsymbol{w}_{(k,i)}^{\lambda}\big),
\end{equation}
where $ \eta > 0 $ denotes the learning rate which is used in SGD. In this paper, $ \eta $ is fixed with a pre-specified value and is equal across all clients.

For each node $ i $, the total gradient on the collection of data samples at this node is 
\begin{equation}
	\label{local-grad}
	\boldsymbol{G}_{(k,i)} \triangleq \frac{1}{\|\boldsymbol{a_{i}}\|_{1}} \sum_{\lambda=0}^{\tau_{(k,i)}-1} a_{i}^{\lambda} \nabla F_{i}\big(\boldsymbol{w}_{(k,i)}^{\lambda}\big),
\end{equation}
where $ \boldsymbol{a}_{i} \in \mathbb{R}^{\tau_{(k,i)}}$ is a non-negative vector and defines how stochastic gradients are locally accumulated, and $ a_{i}^{\lambda} $ is the element with index $ \lambda $ in the vector $ \boldsymbol{a}_{i} $. Under this rule, $ \boldsymbol{G}_{(k,i)} $ denotes a locally normalized gradient at node $ i $.

After one or multiple local SGD iterations, a global step is performed through the parameter server to update the global model parameters, which is
\begin{equation}
	\label{global update}
	\boldsymbol{w}_{k+1} - \boldsymbol{w}_{k} = -\eta \sum_{i=1}^{N} p_{i} \|\boldsymbol{a_{i}}\|_{1} \boldsymbol{G}_{(k,i)}.
\end{equation}
We define $ D_{i} \triangleq \left| \mathcal{D}_{i} \right| $ and $ D \triangleq \sum_{i=1}^{N} D_{i} $, where $ \left| . \right| $ denotes the size of the set, and the weight $ p_i $ is equal to $ D_{i} / D $.

At the end of each round $ k=0,1,2,...,K $, the aggregator estimates the global loss function that we define
\begin{equation}
	\label{eq-global-loss}		
	F\left(\boldsymbol{w}_{k+1}\right) \triangleq \sum_{i=1}^{N} p_{i} F_{i}\big(\boldsymbol{w}_{k}^{i, \lambda=\tau_{i}}\big),
\end{equation}
where $ F_{i}\big(\boldsymbol{w}_{k}^{i, \lambda=\tau_{i}}\big) $ is local loss function which is computed from $ \boldsymbol{w}_{k}^{i, \lambda=\tau_{i}} $ and local training dataset at work node $ i $. It is important to note that the loss function values estimated during the training phase serve merely as references. Accurate assessment of the loss function should be performed directly on the test dataset.

We define that each round $ k $ includes one global update step and $ \tau_{(k,i)} $ local SGD iterations at each node $ i $. In the $ k $-th round, the generalized update rules are shown in Fig. \ref{fedavg}. Later in this section, we will present two algorithms that are based on these general update rules.

\subsection{FedAvg and FedNova Algorithms}
In the FedAvg and FedNova algorithms, each client performs $ E $ epochs (traversals of their local dataset) of local SGD with a mini-batch size $ B $. Thus, if a client has $ D_{i} $ local data samples, the number of local SGD iterations is $ \tau_{(k,i)} = \lfloor E(D_{i}/B) \rfloor $, which can vary widely across clients.
\subsubsection{The FedAvg Algorithm} 
The FedAvg algorithm provided the basic principles for FL in the Non-IID datasets. As mentioned in \cite{wang2019adaptive}, SGD can be seen as an approximation to Deterministic Gradient Descent (DGD). Therefore, for convergence analysis of federated optimization, it is generally assumed that the number of local updates is the same across all clients (that is, $ \tau_{(k,i)} = \tau $ for all clients i).

According to the definition in Section \ref{update rule}, FedAvg has $ \boldsymbol{a}_{i} = [1,1,...,1] $ and $ \|\boldsymbol{a}_{i}\|_{1} = \tau $, the update rule specializes \eqref{local-grad} and \eqref{global update} and is as follows:
\begin{equation}
	\label{eq-fedavg}
	\left\lbrace \begin{aligned}  
		&\boldsymbol{G}_{(k,i)} \triangleq  \sum_{\lambda=0}^{\tau-1} \nabla F_{i}\big(\boldsymbol{w}_{(k,i)}^{\lambda}\big), \\
		&\boldsymbol{w}_{k+1} - \boldsymbol{w}_{k} = -\eta \sum_{i=1}^{N} p_{i} \boldsymbol{G}_{(k,i)}.
	\end{aligned}\right.
\end{equation}
Therefore, FedAvg is a simplified special case of the generalized update rules. And next, we will introduce an algorithm that improves on the generalized update rules.
\subsubsection{The FedNova Algorithm}
\label{fednova-alg}
The FedNova algorithm has the same definition of $ \boldsymbol{a}_{i}  = [1,1,...,1] $ as the FedAvg algorithm, but FedNova does not impose a special constraint on the number of local SGD iterations, thus $ \|\boldsymbol{a_{i}}\|_{1} = \tau_{(k,i)} = \lfloor E(D_{i}/B) \rfloor $. FedNova still uses the local SGD update rule of \eqref{local update}, the locally normalized gradient of \eqref{local-grad} and the global aggregation rule is as follows:
\begin{equation}
	\label{eq-fednova}
	\left\lbrace \begin{aligned}
		&\boldsymbol{G}_{(k,i)} \triangleq \frac{1}{\tau_{(k,i)}} \sum_{\lambda=0}^{\tau_{(k,i)}-1} \nabla F_{i}\big(\boldsymbol{w}_{(k,i)}^{\lambda}\big),\\
	    &\boldsymbol{w}_{k+1} - \boldsymbol{w}_{k} = -\eta \tau_{k} \boldsymbol{d}_{k},
	\end{aligned}\right.
\end{equation}
where $ \tau_{k} \triangleq \sum_{i=1}^{N} p_{i} \tau_{(k,i)} $ is the aggregated value of the number of SGD iterations on each client and $ \boldsymbol{d}_{k} \triangleq \sum_{i=1}^{N} p_{i} \boldsymbol{G}_{(k,i)} $ is the normalized averaging gradient.

Equation \eqref{eq-fednova} can be seen as a decomposition of equation \eqref{global update}, where $ -\boldsymbol{d}_{k} $ denotes the direction of global gradient descent and $ \eta \tau_{k} $ denotes the step size of global gradient descent.

\subsection{Learning Objective and FedVeca Method}
\label{goal}
\begin{figure}[!t]
	\centering
	\includegraphics[width=2.5in]{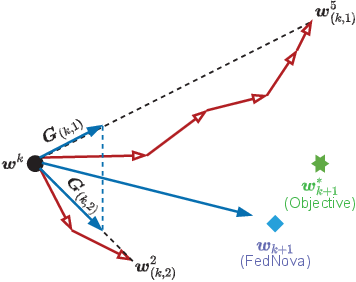}
	\caption{Illustrative diagram of the learning problem in the $ k $-th round.}
	\label{problem}
\end{figure}
In FL, for each global model $ \boldsymbol{w}_{k} $ there is its corresponding loss function $ F(\boldsymbol{w}_{k}) $. The objective of learning is to find optimal global model $ \boldsymbol{w}_{K}^{*} $ to minimize $ F(\boldsymbol{w}_{K}^{*}) $
\begin{equation}
	\boldsymbol{w}_{K}^{*} \triangleq \arg\min  F(\boldsymbol{w}_{K}),
\end{equation}
thus the value of $ F(\boldsymbol{w}_{k}) $ should decrease as $ k $ increases. In the $ k $-th round, we simplify this objective to find an optimal global model $ \boldsymbol{w}_{k+1}^{*} $ that is
\begin{equation}
	\label{our-goal}
	\boldsymbol{w}_{k+1}^{*} \triangleq \arg\min \left\lbrace F(\boldsymbol{w}_{k+1}^{*}) - F(\boldsymbol{w}_{k})\right\rbrace.
\end{equation}

In FedVeca, we use the update rule \eqref{eq-fednova} of FedNova and define the local gradient $ \boldsymbol{G}_{(k,i)} $ as an averaged bi-directional vector whose step size is the value of $ \tau_{(k,i)} $. In addition to that, we allow a heterogeneous value of step size at each client and analyze the relationship between our objective $ \boldsymbol{w}_{k+1}^{*} $ and FedNova's global model $ \boldsymbol{w}_{k+1} $. As shown in Fig. \ref{problem}, we assume that there are two clients, the Client 1 with 5 local SGD iterations (step size of 5) and the Client 2 with 2 local SGD iterations (steo size of 2). On these two clients, the local gradient $ \boldsymbol{G}_{(k,1)} $ on Client 1 is more closely aligned with the direction of $ \boldsymbol{w}_{k} - \boldsymbol{w}_{k+1}^{*} $, making it a positive vector. In contrast, the local gradient on Client 2 deviates more significantly from the direction of $ \boldsymbol{w}_{k} - \boldsymbol{w}_{k+1}^{*} $, making it a negative vector. The mathematical quantification underlying the classification of these two types of vectors will be presented in detail in Theorem \ref{theorem-2} of Section \ref{conv-ana}. The direction of the global gradient $ \boldsymbol{w}_{k} - \boldsymbol{w}_{k+1} $ is then determined by the sum of these two vectors.

In order to make the global model $ \boldsymbol{w}_{k+1} $ close to our objective $ \boldsymbol{w}_{k+1}^{*} $, two factors need to be controlled: direction and step size of the global gradient. According to the definition of $ d_{k} $ and $ \tau_{k} $ in Section \ref{fednova-alg}, we know that the averaged bi-directional vector $ \boldsymbol{G}_{(k,i)} $ and it's step size $ \tau_{(k,i)} $ control the direction and step size of the global gradient, respectively. Moreover, we can see in \eqref{eq-fednova} that there is a certain connection between $ \boldsymbol{G}_{(k,i)} $ and $ \tau_{(k,i)} $, thus a certain connection between the global model $ \boldsymbol{w}_{k+1} $ and the bi-directional vector's step size $ \tau_{(k,i)} $.

We set that $ \tau_{(k,i)} $ satisfy $ \tau_{(k,i)} \neq \lfloor E(D_{i}/B) \rfloor $ and can be different in each round $ k $ and at each client $ i $. Therefore, a natural question is how to determine the optimal values of $ \tau_{(k,i)} $ at each node $ i $, so that to get the global objective $ \boldsymbol{w}_{k+1}^{*} $. Based on this, if we want to get the optimal $ \boldsymbol{w}_{K}^{*} $, we need adaptive control of $ \tau_{(k,i)} $ for each round $ k $. And next, we will describe how our FedVeca algorithm achieves this objective.

\section{Convergence Analysis and FedVeca Algorithm}
\label{sec-analysis}
In this section, we first analyze the convergence of FedVeca and get an upper bound of $ \tau_{(k,i)} $. Then, we use this upper bound to design our algorithm to adaptively control the value of  $ \tau_{(k,i)} $ in each round $ k $ and at each word node $ i $.
\subsection{Convergence Analysis}
\label{conv-ana}
We analyze the convergence of FedVeca in this subsection and find an upper bound of $ F(\boldsymbol{w}_{k+1}) - F(\boldsymbol{w}_{k}) $. To facilitate the analysis, we first introduce the definition of the global gradient $ \nabla F(\boldsymbol{w}_{k}) $ which is the optimal gradient of $ \boldsymbol{w}^{*}_{k+1} - \boldsymbol{w}_{k} $ on all training datasets, and we have $ \left\|\boldsymbol{d}_{k}\right\| \leq \left\|\nabla F\left(\boldsymbol{w}_{k}\right)\right\| $ by comparing \eqref{eq-fednova} and \eqref{our-goal}. However, the global gradient $ \nabla F(\boldsymbol{w}_{k}) $ can not be calculated directly from all training datasets in the $ k $-th round of FL. Thus, at the end of each round $ k=1,2,...,K $, the server estimates the last global gradient that we define
\begin{equation}
	\label{eq-global grad}		
	\nabla F(\boldsymbol{w}_{k-1}) \triangleq \sum_{i=1}^{N} p_{i} \nabla F_{i}\left(\boldsymbol{w}_{k-1}\right),
\end{equation}
where $ \nabla F_{i}\left(\boldsymbol{w}_{k-1}\right) $ is the local gradient of $ \boldsymbol{w}_{k-1} $ at client $ i $ (can be calculated directly from the local training dataset). We assume the global gradient $ \nabla F(\boldsymbol{w}_{k-1}) $ is convex and all the assumptions in this paper are as follows:
\begin{assumption}[Global Lipschitz smooth]
	\label{assumption-L}
	Each global gradient is convex and Lipschitz smooth for $ k \in [0,k-1] $, that is,
	\begin{equation}
		\| \nabla F(\boldsymbol{w}_{k+1}) - \nabla F(\boldsymbol{w}_{k}) \| \leq L \| \boldsymbol{w}_{k+1} - \boldsymbol{w}_{k} \|. 
	\end{equation}
\end{assumption}

\begin{assumption}
	\label{assumption-d}
	For $ k \in [0,k-1] $, we have
	\begin{equation}
		\left\|\boldsymbol{d}_{k}\right\| \leq \left\|\nabla F\left(\boldsymbol{w}_{k}\right)\right\|.
	\end{equation}
\end{assumption}

\begin{assumption}[Local Lipschitz smooth]
	\label{assumption-beta}
	For any local gradient $ \nabla F_{i}\left(\boldsymbol{w}_{(k,i)}^{\lambda}\right) $ and $ \lambda \in [0, \tau_{(k, i)}-1] $, there
	exist constants $ \beta_{(k,i)} \geq 0 $ such that 
	\begin{equation}
		\big\|\nabla F_{i}\left(\boldsymbol{w}_{k}\right)-\nabla F_{i}\big(\boldsymbol{w}_{(k, i)}^{\lambda}\big)\big\| \leq \beta_{(k,i)}\big\|\boldsymbol{w}_{k}-\boldsymbol{w}_{(k, i)}^{\lambda}\big\|.
	\end{equation}
\end{assumption}

\begin{assumption}[Bounded gradient error]
	\label{assumption-delta}
	For all local gradients, $ s \in [0, \lambda] $ and $ \lambda \in [1, \tau_{(k, i)}-1] $, there
	exist constants $ \delta_{(k,i)} \geq 0 $ such that 
	\begin{equation}
		\big\|\sum_{s=0}^{\lambda-1} \nabla F_{i}\left(\boldsymbol{w}_{(k, i)}^{\lambda}\right)\big\|^{2} \leq \delta_{(k,i)} \sum_{s=0}^{\lambda-1}\big\|\nabla F(\boldsymbol{w}_{k})\big\|^{2}.
	\end{equation}
\end{assumption}
Under Assumptions \ref{assumption-L} to \ref{assumption-delta} and the update rule of FedVeca, we get the upper bound of $ F(\boldsymbol{w}_{k+1}) - F(\boldsymbol{w}_{k}) $.
\begin{theorem}
	\label{theorem-1}
	In the $ k $-th round for $ k \in [0,k-1] $, when $ \eta \tau_{k} L -1 \geq 0 $, we have
	\begin{equation}
		\label{F-theo1}
		\begin{aligned}
			&\frac{F\left(\boldsymbol{w}_{k+1}\right)-F\left(\boldsymbol{w}_{k}\right)}{\left\|\nabla F\left(\boldsymbol{w}_{k}\right)\right\|^{2}} \\
			\leq &\eta \tau_{k}\big(\frac{1}{4} \eta \sum_{i=1}^{N} p_{i}\left(\tau_{(k,i)}\left(2 L+A_{(k,i)} \right)-A_{(k,i)}\right) -1 \big),
		\end{aligned}
	\end{equation}
where $ A_{(k,i)} \triangleq \eta \beta_{(k,i)}^{2} \delta_{(k,i)} $ is a variable that varies with $ \delta_{(k,i)} $ and $ \beta_{(k,i)}^{2} $.
\end{theorem}
\begin{proof}
	We first perform an inequality transformation on $ F\left(\boldsymbol{w}_{k+1}\right)-F\left(\boldsymbol{w}_{k}\right) $ and then convert the equation to contain only the similar items of $ \left\|\nabla F\left(\boldsymbol{w}_{k}\right)\right\|^{2} $. For details, see Appendix A 
\end{proof}
According to Theorem \ref{theorem-1}, if the model can converge under our method, then it should be consistent with $ F\left(\boldsymbol{w}_{k+1}\right) - F\left(\boldsymbol{w}_{k}\right) < 0 $. It is equivalent to $ \frac{1}{4} \eta \sum_{i=1}^{N} p_{i}\left(\tau_{(k,i)}\left(2 L+A_{(k,i)} \right)-A_{(k,i)}\right) < 1 $, and this inequality is related to $ \tau_{(k,i)} $. We can see from \eqref{F-theo1} that the parameters in $ A_{(k,i)} $ (the effect of Non-IID) do not play a role when $ \tau_{(k,i)} = 1 $. Therefore, we set the lower bound of $ \tau_{(k,i)} > 1 $, and have another theorem to determine the upper bound on $ \tau_{(k,i)} $.

\begin{theorem}
	\label{theorem-2}
		In the $ k $-th round for $ k \in [0,k-1] $, the model converges when
		\begin{equation}
			\label{eq-theorem-2}
			 \tau_{(k,i)} \leq \frac{A_{(k, i)}}{A_{(k, i)}-\alpha_{k} \min _{i} A_{(k, i)}}.
		\end{equation}
	for each client $ i $, where $ \alpha_{k} $ is a real number and has $ \alpha_{k} \in (0, \frac{2L}{\min_{i} A_{(k,i)}}) $ (when $ \frac{2L}{\min_{i} A_{(k,i)}} < 1 $) and $ \alpha_{k} \in (0,1) $ (when $ \frac{2L}{\min_{i} A_{(k,i)}} > 1 $).
\end{theorem}
\begin{proof}
	For details, see Appendix B 
\end{proof}

According to \eqref{eq-theorem-2}, we can define the positive and negative direction of the bid-irectional vector (in Section \ref{goal}) based on the gap between the value of $ A_{(k,i)} $ and the value of $ \alpha_{k} \min_{i} A_{(k,i)} $. Further, we can obtain the relationship between the bounds on the step size $ \tau_{(k,i)} $ and the corresponding direction of the bi-directional vector. We next describe how our FedVeca algorithm adaptively controls $ \tau_{(k,i)} $ to achieve our objective in Section \ref{goal}.

\subsection{FedVeca Algorithm}
FL is required to perform more local computations in each communication round \cite{mcmahan2017communication}. Thus, according to the convex assumption and Theorem \ref{theorem-2}, we predict that
\begin{equation}
	\label{tau}
	\tau_{(k + 1,i)} =\lfloor \frac{A_{(k,i)}}{A_{(k,i)}- \alpha_{k} \min_{i} A_{(k,i)}} \rfloor.
\end{equation}
That is, we use the results of the previous round to predict the number of local SGD iterations on each node in the next round. When $ \tau_{(k + 1,i)} $ is calculated as $ \tau_{(k + 1,i)} = 1 $, to keep $  \tau_{(k + 1,i)} > 1 $, we reset $ \tau_{(k + 1,i)} = 2 $ in our algorithm. Moreover, when $ \min_{i} A_{(k,i)} $ is determined in the $ k $-th round, we can choose a suitable value of $ \alpha_{k} $ to make the global model for the next round close to our global objective (which is mentioned in Section \ref{goal}).

As mentioned earlier, the local update runs on the clients and the global aggregation is performed with the assistance of the parameter server. The complete process of the parameter server and each client is presented in Algorithm \ref{alg-server} and Algorithm \ref{alg-node}, respectively, where Lines 5-8 of Algorithm \ref{alg-node} are local updates and the rest are considered as part of initialization, global aggregation and computing the value of $ \tau_{(k + 1,i)} $. We assume that the server initiates the learning process, then initializes $ \tau_{(0,i)} $, $ \boldsymbol{w}_{0} $ and $ \nabla F\left(\boldsymbol{w}_{0}\right) $ and sends it to all clients. The input consists of a given $ K $, an $ \alpha_{k} $ that transforms with round $ k $, and a $ \eta $ that is fixed for all rounds, and finally FedVeca algorithm gives the global model $ \boldsymbol{w}_{K} $ obtained for the final round $ K $.
\begin{figure}[!t]
	\begin{algorithm}[H]
		\caption{Procedure at the server.}
		\label{alg-server}
		\begin{algorithmic}[1]
			\REQUIRE Total round $ K $, parameter $ \alpha_{k} $, learning rate $ \eta $ and minimum loss value $F_{m}$
			\ENSURE $ \boldsymbol{w}_{K} $
			\STATE Initialize $ k = 0 $, $F_{m}=\infty$, $ \tau_{(k,i)} $ and $ \boldsymbol{w}_{k} $;
			\REPEAT
			\STATE Send $ \tau_{(k,i)} $ and $ \boldsymbol{w}_{k} $ to all clients;
			\STATE Receive $ \nabla F_{i}(\boldsymbol{w}_{k}) $ and $F_{i}\big(\boldsymbol{w}_{(k,i)}^{\lambda=\tau_{i}}\big)$ from each node $ i $;
			\STATE Compute $ \nabla F\left(\boldsymbol{w}_{k}\right) $ and $F\left(\boldsymbol{w}_{k+1}\right)$ according to \eqref{eq-global grad} and \eqref{eq-global-loss}, respectively;
			\STATE Receive $ \boldsymbol{G}_{(k,i)} $ from each node $ i $;
			\STATE Compute $ \boldsymbol{w}_{k+1} $ according to \eqref{eq-fednova};
                \IF{$F\left(\boldsymbol{w}_{k+1}\right)\leq F_{m}$}
                \STATE Set $F_{m}\leftarrow F\left(\boldsymbol{w}_{k+1}\right)$
                \ELSE
                \STATE Set $\boldsymbol{w}_{k+1}\leftarrow\boldsymbol{w}_{k}$
                \ENDIF
			\IF{$ k \geq 1 $}
			\STATE Send $ \nabla F\left(\boldsymbol{w}_{k-1}\right) $ to all clients;
			\STATE Receive $ \beta_{(k,i)} $ and $ \delta_{(k,i)} $;
			\IF{$ k = 1 $}
			\STATE Estimate $ L_{k-1} \leftarrow \| \nabla F(\boldsymbol{w}_{k-1}) \| / \| \boldsymbol{w}_{k-1} \| $;
			\ELSE
			\STATE Estimate $ L_{k-1} \leftarrow \| \nabla F(\boldsymbol{w}_{k-1}) - \nabla F(\boldsymbol{w}_{k-2}) \| / $\\ $ \| \boldsymbol{w}_{k-1} - \boldsymbol{w}_{k-2} \| $;
			\ENDIF
			\STATE Estimate $ L \leftarrow \max_{k} L_{k-1} $;
			\STATE Compute the value of $ \tau_{(k + 1,i)} $ according to \eqref{tau};
			\STATE Compute $ \tau_{k+1} $ according to \eqref{eq-fednova};
			\IF{$ \tau_{(k + 1,i)} \leq 1 $}
			\STATE Set $ \tau_{(k + 1,i)} = 2 $;
			\ENDIF
			\ENDIF
			\STATE $ k \leftarrow k+1 $;
			\IF{$ k = 1 $}
			\STATE Set $ \tau_{(k,i)} = \tau_{(k - 1,i)} $;
			\ENDIF
			\IF{$ k = K $}
			\STATE Set $ STOP $ flag, and send it to all clients;
			\ENDIF			\UNTIL{$ STOP $ flag is set}
		\end{algorithmic}
	\end{algorithm}
\end{figure}
\begin{figure}[!t]
	\begin{algorithm}[H]
		\caption{Procedure at each client $ i $.}
		\label{alg-node}
		\begin{algorithmic}[1]
			\STATE Initialize $ \lambda \leftarrow 0 $, $ k = 0 $ and $ \boldsymbol{w}_{(k, i)}^{\lambda} $;
			\REPEAT
			\STATE Receive $ \tau_{(k,i)} $ and $ \boldsymbol{w}_{k} $ from the server;
			\STATE Set $ \boldsymbol{w}_{(k, i)}^{\lambda} \leftarrow \boldsymbol{w}_{k} $;
			\FOR{$ \lambda = 0,1,2,...,\tau_{(k,i)} -1 $}
			\STATE Compute $ \nabla F_{i}\big(\boldsymbol{w}_{(k,i)}^{\lambda} \big) $;
			\STATE Perform local update of $ \boldsymbol{w}_{(k,i)}^{\lambda} $ according to \eqref{local update};
			\ENDFOR
			\STATE Compute $ \nabla F_{i}\left(\boldsymbol{w}_{k}\right) $ and $F_{i}\big(\boldsymbol{w}_{(k,i)}^{\lambda=\tau_{i}}\big)$;
			\STATE Send $ \nabla F_{i}\left(\boldsymbol{w}_{k}\right) $ and $F_{i}\big(\boldsymbol{w}_{(k,i)}^{\lambda=\tau_{i}}\big)$ to parameter server;
			\STATE Compute $ \boldsymbol{G}_{(k,i)} $ according to \eqref{eq-fednova};
			\STATE Send $ \boldsymbol{G}_{(k,i)} $ to parameter server;
			\IF{$ k \geq 1 $}
			\STATE Receive $ \nabla F\left(\boldsymbol{w}_{k-1}\right) $ from the server;
			\STATE Estimate $ \beta_{(k,i)}^{\lambda} \leftarrow \big\|\nabla F_{i}\left(\boldsymbol{w}_{k}\right)-\nabla F_{i}\big(\boldsymbol{w}_{(k, i)}^{\lambda}\big)\big\| / $\\ $ \big\|\boldsymbol{w}_{k}-\boldsymbol{w}_{(k, i)}^{\lambda}\big\|$ for $ \lambda \in [0, \tau_{(k, i)}-1] $;
			\STATE Estimate $ \beta_{(k,i)} \leftarrow \max_{\lambda} \beta_{(k,i)}^{\lambda} $;
			\STATE Estimate $ \delta_{(k,i)}^{\lambda} \leftarrow \big\|\sum_{s=0}^{\lambda} \nabla F_{i}\big(\boldsymbol{w}_{(k, i)}^{s}\big)\big\|^{2} / $ \\ $ \sum_{s=0}^{\lambda}\left\|\nabla F(\boldsymbol{w}_{k-1})\right\|^{2} $ for $ \lambda \in [1, \tau_{(k, i)}-1] $;
			\STATE Estimate $ \delta_{(k,i)} \leftarrow \max_{\lambda} \delta_{(k,i)}^{\lambda} $;	
			\STATE Send $ \beta_{(k,i)} $ and $ \delta_{(k,i)} $ to parameter server;		
			\ENDIF
			\STATE $ k \leftarrow k + 1 $;
			\UNTIL{$ STOP $ flag is received}
		\end{algorithmic}
	\end{algorithm}
\end{figure}
\subsubsection{Handling of SGD}
When using SGD with FedVeca algorithm at all clients, all their gradients are computed on mini-batches. Each SGD step corresponds to a model update step where the gradient is computed on a mini-batch of local training data in Line 6 of Algorithm \ref{alg-node}. The mini-batch changes for every step of the local iteration, i.e., for each new local iteration, a new mini-batch of a given size is randomly selected from the local training data.  After $ \tau_{(k, i)} $ local update steps, we average the corresponding model gradient vectors on these training data and get obtain the direction vector $ G_{(k,i)} $ of SGD at node $ i $ (Line 11 of Algorithm \ref{alg-node}). 

When the parameter server receives $ G_{(k,i)} $ from each client, it updates $ \boldsymbol{w}_{k+1} $ (Line 7 of Algorithm \ref{alg-server}) according to the second term in \eqref{eq-fednova} where $ \tau_{k} $ is calculated based on $ \tau_{(k, i)} $ in Line 23 of Algorithm \ref{alg-server}. With the value of $F_{i}\big(\boldsymbol{w}_{(k,i)}^{\lambda=\tau_{i}}\big)$ received from the clients (Line 4 of Algorithm \ref{alg-server}) and Equation \eqref{eq-global-loss}, we can estimate the value of the loss function of $ \boldsymbol{w}_{k+1} $ (Line 5 of Algorithm \ref{alg-server}). Compare it with the set minimum loss function value $F_m$, if it is less than, then accept this global update and assign the value to $F_m$, and if not, then do not accept this global update (Line 8-12 of Algorithm \ref{alg-server}). When the program proceeds to the set number of rounds $ k $, the final model parameter $ \boldsymbol{w}_{K} $ is obtained at the server in Lines 32-34 of Algorithm \ref{alg-server}. Then we set the $ STOP $ flag to stop the server-side program and send the flag to all clients to stop the local program.
\subsubsection{Estimation of Parameters}
\label{sec-estmation}
According to Assumptions \ref{assumption-L}, \ref{assumption-beta} and \ref{assumption-delta}, we have three parameters $ L $, $ \beta_{(k, i)} $ and $ \delta_{(k, i)} $ that need to be estimated in real-time during the learning process. $ L $ is the parameter describing the smooth of global gradients and should be ensured to satisfy Assumption \ref{assumption-L}, so we perform the estimation of its maximum value (in all rounds for $ k>0 $) on the server side in Lines 21 of Algorithm \ref{alg-server}. Estimating $ L $ requires the global gradient of the previous round, which does not exist when $ k=0 $, so we give this program a delay of one round in Lines 19 of Algorithm \ref{alg-server}. We know that $ L $ is not involved in the calculation of the new $ \tau_{(k, i)} $ in \eqref{tau} and used to satisfy that premise $ \eta \tau_{k} L -1 \geq 0 $ of Theorem \ref{theorem-1}, so a delay of one round in its estimation $ L_{k-1} $ has little effect on the final result. 

Therefore FedVeca algorithm should estimate the value of $ A_{(k,i)} $. The expression of $ A_{(k,i)} $ includes parameters $ \beta_{(k, i)} $ and $ \delta_{(k, i)} $ which need to be estimated in practice on the client side. Before starting the estimation, we need to calculate $ F_{i}\left(\boldsymbol{w}_{k}\right) $ on the local training dataset in Line 9 of Algorithm \ref{alg-node}. Then, the estimation of $ \beta_{(k, i)} $ is computed from $ \lambda = 0 $ to $ \lambda = \tau_{(k, i)} -1 $ and the estimation of $ \delta_{(k, i)} $ is computed from $ \lambda = 1 $ to $ \lambda = \tau_{(k, i)} -1 $ in Lines 15 and 17 of Algorithm \ref{alg-node}. Finally, the largest estimations of $ \beta_{(k, i)} $ and  $ \delta_{(k, i)} $ from these iterations will be selected (Lines 16 and 18 of Algorithm \ref{alg-node}) and sent to the parameter server when $ k > 0 $ (Lines 19 of Algorithm \ref{alg-node}). 

\subsubsection{Computing $ \tau_{(k + 1,i)} $}
When $ k=0 $, the server does not receive the relevant parameters ($ \beta_{(k, i)} $ and $ \delta_{(k, i)} $ ) from the clients to calculate the value of $ \tau_{(k + 1,i)} $. This is because the value of $ L $ cannot be estimated at this point, and we have no way to choose the value of $ \alpha_{k} $ to input. When $ k \geq 1 $, the value of $ \tau_{(k + 1,i)} $ can be computed according to \eqref{tau} in Line 22 of Algorithm \ref{alg-server}. Then, to satisfy $ \tau_{(k + 1,i)} > 1 $ in Section \ref{conv-ana}, we set $ \tau_{(k+1,i)} = 2 $ for $ \tau_{(k + 1,i)} \leq 1 $ in Lines 24-26 of Algorithm \ref{alg-server}.

After computing the value of $ \tau_{(k+1,i)} $, we can compute the global step size $ \tau_{k+1} $ for the next round ($ k \leftarrow k + 1 $) of aggregation in Line 23 of Algorithm \ref{alg-server}. Combining our obtained estimation of $ L $, we can verify the premise $ \eta \tau_{k} L -1 \geq 0 $ of Theorem \ref{theorem-1}. The verification results and the test results of FedVeca algorithm will be shown in Section \ref{sec-experiment}.

\section{Performance Analysis}
\label{sec-experiment}
To evaluate FedVeca algorithm, the experiments focused on the qualitative and quantitative analysis of general properties under the setup of our prototype system and the simulation of IID and Non-IID datasets. 
\subsection{Setup}
\label{subsec-set}
We first evaluate the general performance of FedVeca algorithm, thus we conducted experiments on a networked prototype system with five clients. The prototype system consists of five Raspberry Pi (version 4B) devices and one laptop computer, which are all interconnected via Wi-Fi in an office building. The laptop computer has an aggregator and implements FedVeca algorithm of parameter server, and the Raspberry Pi device implements FedVeca algorithm of client. All of these five clients have model training with local datasets and different simulated dataset distributions of IID and Non-IID.

\subsubsection{Baselines}
We compare FedVeca method with the following baseline approaches: 
\begin{itemize}
	\item Centralized SGD where the entire training dataset is stored on a single device and the model is trained directly on that device using a standard (centralized) SGD procedure.
	\item Standard FL approach which uses FedAvg algorithm and has the fixed (non-adaptive) value of $ \tau_{(k, i)} $ at all clients in all rounds.
	\item Novel FL approach which uses FedNova algorithm and has the same value of $ \tau_{(k, i)} $ as FedAvg algorithm at all clients in all rounds.
        \item Novel FL approach which uses FedProx algorithm and has the same value of $ \tau_{(k, i)} $ as FedAvg algorithm at all clients in all rounds.
        \item Novel FL approach which uses SCAFFOLD algorithm and has the same value of $ \tau_{(k, i)} $ as FedAvg algorithm at all clients in all rounds.
\end{itemize}

For a fair comparison, we first run FedVeca algorithm and then record the total number of local iterations $ \tau_{all} $ run by all nodes in all $ K $ rounds. Then for the centralized SGD, we train $ \tau_{all} $ iterations and each iteration randomly picks a batch of the same size $ B $ (fixed) as FedVeca algorithm, then we use this model for evaluating the convergence of other trained models. When evaluating the FedAvg algorithm and FedNova algorithm, we compute the average epoch $ E_{avg} = \tau_{all} / K \times B / D  $ of all rounds, and assign a fixed $ \tau_{(k, i)} = \lfloor E_{avg}(D_{i}/B) \rfloor $ to the clients for each round.

\subsubsection{Model and Datasets}
\label{model-data}
We evaluate the training of two different models on two different datasets, which represent both small and large models and datasets. The models include squared-Support Vector Machine (SVM)\footnote{The squared-SVM has a fully connected neural network, and outputs a binary label that corresponds to whether the digit is even or odd.} (we refer to as SVM in short in the following) and deep Convolutional Neural Network (CNN)\footnote{The CNN has two $ 5 \times 5 \times 32 $ convolution layers, two $ 2 \times 2 $ MaxPoll layers, a $ 1568 \times 256 $ fully connected layer, a $ 256 \times 10 $ fully connected layer, and a softmax output layer with 10 units.}. Among them, the loss functions for SVM satisfy Assumption \ref{assumption-L}, whereas the loss functions for CNN are non-convex and thus do not satisfy Assumption \ref{assumption-L}.

SVM is trained on the original MNIST (which we refer to as MNIST in short in the following) dataset \cite{lecun1998gradient}, which contains the gray-scale images of $ 7 \times 10^{4} $ handwritten digits ($ 6 \times 10^{4} $ for training and $ 10^{4} $ for testing).

CNN is trained using SGD on two different datasets: the MNIST dataset and the CIFAR-10 dataset \cite{krizhevsky2009learning}, and the CIFAR-10 dataset includes $ 6 \times 10^{4} $ color images ($ 5 \times 10^{4} $ for training and $ 10^{4} $ for testing) associated with a label from 10 classes. A separate CNN model is trained on each dataset, to perform multi-class classification among the 10 different labels in the dataset.
\subsubsection{Simulation of Dataset Distribution}
\label{sec-distribute}
To simulate the dataset distribution, we set up two different Non-IID cases and a standard IID case.
\begin{itemize}
	\item \textbf{Case 1} (IID): Each data sample is randomly assigned to a client, thus each client has a uniform (but not full) information.
	\item \textbf{Case 2} (Non-IID): All the data samples in each client have the same label, which means that the dataset on each node has its unique features.
	\item \textbf{Case 3} (Non-IID): Data samples with the first half of the labels are distributed to the first half of the clients as in Case 1, the other samples are distributed to the second half of the clients as in Case 2.
\end{itemize}
\subsubsection{Training and Control Parameters}
\label{control}
In all our experiments, we set the maximum $ \tau_{(k,i)} $ value $ \max \tau_{(k,i)} = 50 $ to reduce the impact of errors in the experiment. Unless otherwise specified, we set the control parameter $ \alpha_{k} $ to a fixed value $ \alpha_{k} = 0.95 $ for all rounds. And we manually select the total round $ K =100 $ and the learning rate fixed at $ \eta = 0.01 $, which is acceptable for our learning process. Except for the instantaneous results, the others are the average results of 10 independent experiment runs.

\subsection{Results}
\label{results}
\subsubsection{Loss and Accuracy Values}
\label{loss and acc}
\begin{figure*}[!t]
	\centering
	\includegraphics[width=6.5in]{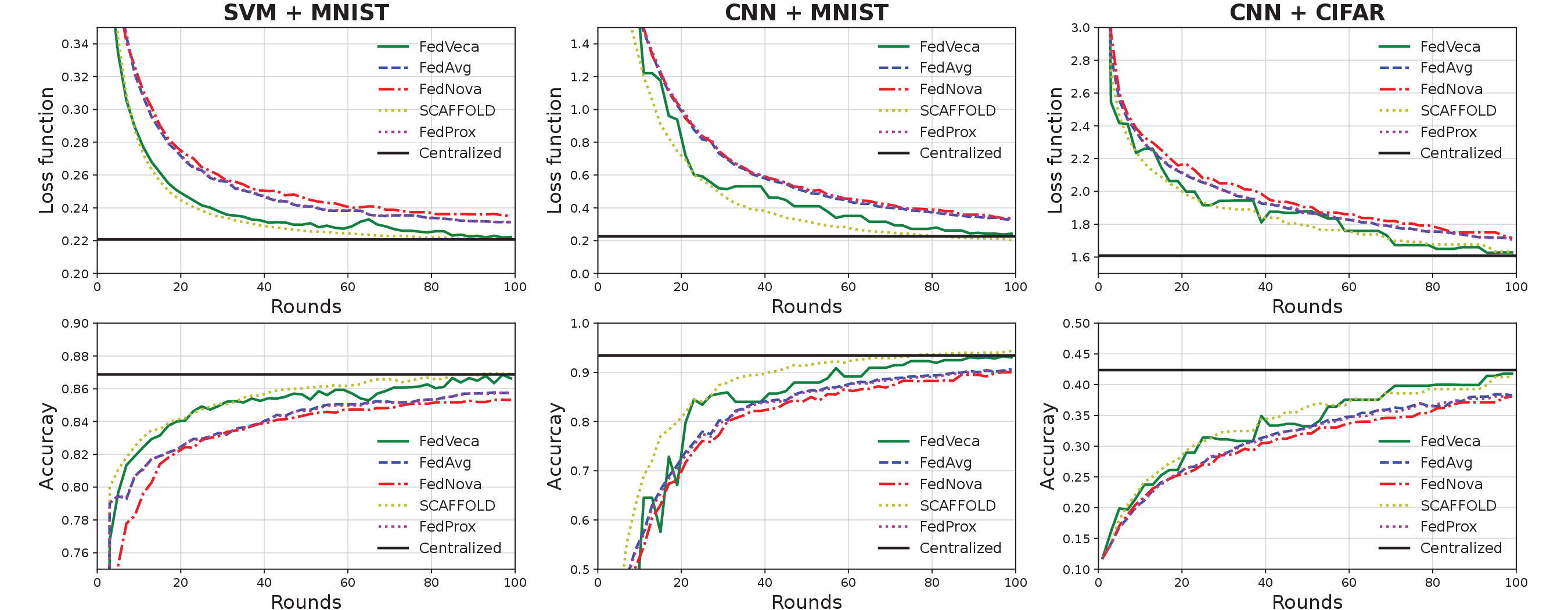}
	\caption{Loss function values and classification accuracy values in Case 3}. The curves show the results from FedVeca algorithm and the baselines with the different models and datasets \cite{lecun1998gradient,krizhevsky2009learning}.
	\label{fig-loss-acc}
\end{figure*}

\begin{table}[!t]
	\centering
	\caption{Model loss values and accuracies for each baseline algorithm obtained in \textbf{Case 3} at the end of the set training rounds ($K=100$).}
	\label{tab:loss-acc-c3}
	\centering
	\begin{tabular}{cccccc}
		\hline
		Model & Algorithm & Loss($ \times 10^{-2}$) & Acc (\%) & Centralized\\
		\hline
		~ & FedAvg & 23.13 & 85.75 & Loss \\
		~ SVM & FedNova & 23.44 & 85.32 & \& Acc \\
            ~ + & FedProx & 23.14 & 85.76 &  \\
            ~ MNIST & \bf{SCAFFOLD} & \bf{22.07} & \bf{86.92} & 22.06 \\
            ~ & FedVeca & 22.22 & 86.63 & \& 86.87 \\
		\hline
		~ & FedAvg & 32.32 & 90.67 & Loss \\
		~ CNN & FedNova & 33.05 & 90.37 & \& Acc \\
            ~ + & FedProx & 32.76 & 90.47 &  \\
            ~ MNIST & \bf{SCAFFOLD} & \bf{20.80} & \bf{93.86} & 22.67 \\
            ~ & FedVeca & 24.25 & 93.04 & \& 93.43 \\
		\hline
		~ & FedAvg & 170.50 & 38.62 & Loss \\
		~ CNN & FedNova & 171.63 & 38.10 & \& Acc \\
            ~ + & FedProx & 170.77 & 38.17 &  \\
            ~ CIFAR & SCAFFOLD & 163.28 & 41.23 & 160.88 \\
            ~ & \bf{FedVeca} & \bf{160.32} & \bf{42.41} & \& 42.38 \\
		\hline
	\end{tabular}
\end{table}

In our first set of experiments, the SVM and CNN models were trained on the prototype system with Case 3, and the results $ \boldsymbol{w}_{k} $ of each round will be calculated on the test dataset with the loss function values and prediction accuracy values which we refer to as loss and accuracy in short in the following.

We record the loss and accuracy on the SVM and CNN classifiers with FedVeca algorithm (with adaptive $ \tau_{(k,i)} $), and compare them to baseline approaches, where the centralized case only has one optimal value as the training result and we show a flat line across different rounds for the ease of comparison, the results are shown in Fig. \ref{fig-loss-acc}.

It can be observed from Fig. \ref{fig-loss-acc} that the curves of loss and accuracy values of FedVeca algorithm fluctuate less on the SVM model, while they have larger fluctuations on the CNN model. This is because the loss function of the SVM model satisfies Assumption \ref{assumption-L} of a convex function and the loss function of the CNN model is non-convex, as we mentioned in Section \ref{model-data}. Since the loss function of the CNN model is non-convex, the variation (orientation and size) of the local model parameters are unstable when making updates, so it is unreliable for us to describe the overall variation of global model parameters with the estimated values of $ L $, $ \beta_{(k,i)} $ and $ \delta_{(k,i)} $. Further, it is less useful to predict the number of local SGD iterations for the next round based on the new value of $ \tau_{(k, i)} $ calculated from these estimated values. However, within the specified number of rounds ($ K = 100 $), FedVeca algorithm is the first to reach the loss and accuracy of the centralized SGD approach compared to FedAvg, FedNova and FedProx on all models and datasets (nearly simultaneous with SCAFFOLD), demonstrating the convergence and faster convergence speed of FedVeca algorithm with Non-IID datasets. 

As demonstrated in Table \ref{tab:loss-acc-c3}, on the CIFAR dataset, the FedVeca algorithm outperforms the FedAvg, FedNova, FedProx and SCAFFOLD algorithm in terms of loss and accuracy improvement in the final trained model, thereby underscoring the effectiveness of the algorithm. As for the MNIST dataset, We elaborate that although FedVeca exhibits marginally lower performance compared to the state-of-the-art algorithm SCAFFOLD, it nonetheless demonstrates a substantial performance advantage over other well-established algorithms, including FedAvg, FedNova, and FedProx. This underscores FedVeca's reliability and its capacity to deliver competitive results across various federated learning scenarios, even when compared with advanced methods.

\subsubsection{Premise of Theorem \ref{theorem-1}}
\label{premise}
\begin{figure}[!t]
	\centering
	\includegraphics[width=2.5in]{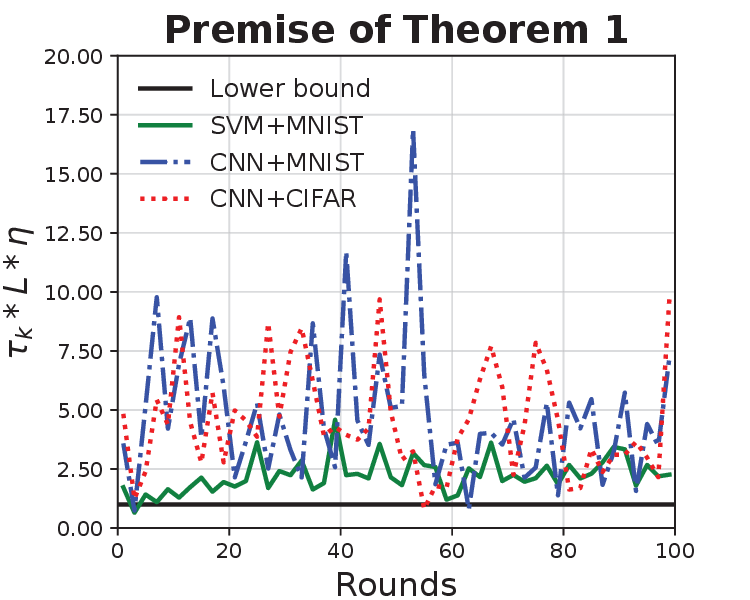}
	\caption{The value of $ \eta \tau_{k} L  $ in Case 3. The curves show the results from FedVeca algorithm with the different models and datasets \cite{lecun1998gradient,krizhevsky2009learning}.}
	\label{fig-premise}
\end{figure}
As described in Section \ref{conv-ana}, the premise for Theorem \ref{theorem-1} to be valid is that $ \eta \tau_{k} L -1 \geq 0 $. Therefore, we record the value of $ \eta \tau_{k} L $ for each round within the specified number of rounds ($ K = 100 $) on all specified models and datasets in Case 3, the results are shown in Fig. \ref{fig-premise}.

In Fig. \ref{fig-premise}, we set the value `$ 1 $' as a lower bound on the value of $ \eta \tau_{k} L $ and use a straight line to represent. Moreover, according to the discussion in Section \ref{sec-estmation}, the estimated value of $ L_{k} $ in Algorithm \ref{alg-server} is delayed by one round, so we make a blank for $ k = 0 $ in Fig. \ref{fig-premise}. As we can see, the specified models and datasets satisfy the premise of Theorem \ref{theorem-1} for all $ K $ rounds, except for the SVM model whose the values of $ \eta \tau_{k} L $ are slightly smaller than the lower bound for the first few rounds on the MNIST dataset (which is within the estimation margin of error). And the SVM model has the most stable $ \eta \tau_{k} L $ values on the MNIST dataset compared to the other specified models and datasets, which is consistent with our conclusion that non-convex loss functions cannot be described accurately as discussed in Section \ref{loss and acc}.

Due to the high complexity of evaluating CNN models, we focus on the SVM model in the following and provide further insights on the prototype system.

\subsubsection{Dataset Distribution}
\label{sec-distribution}
\begin{figure}[!t]
	\centering
	\includegraphics[width=3.5in]{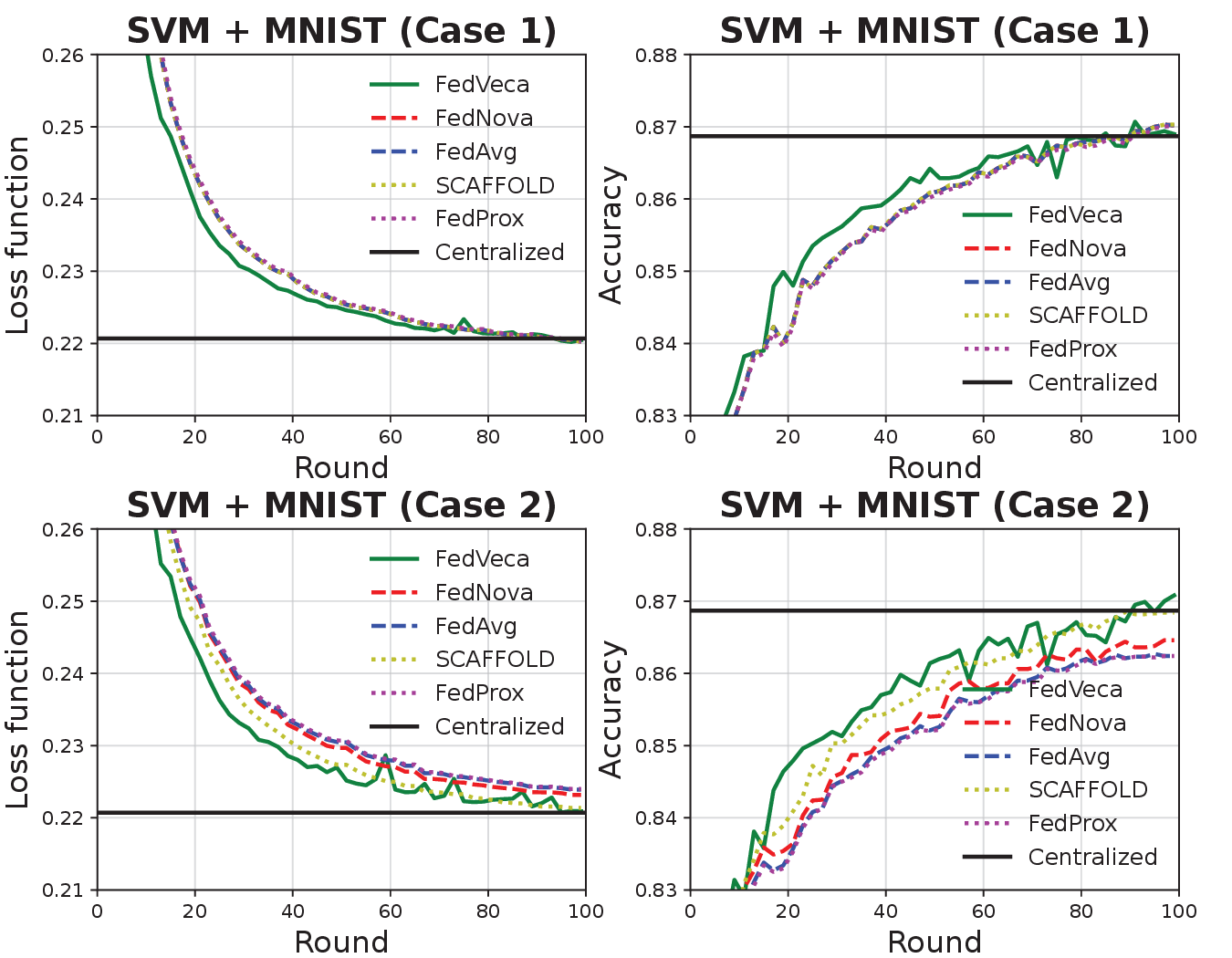}
	\caption{Loss function and classification accuracy on the SVM model and the MNIST dataset \cite{lecun1998gradient} with dataset distriburion of Case 1 and Case 2.}
	\label{fig-case1&case2}
\end{figure}
We test the MNIST dataset distribution on the SVM model for the other two simulations which are Case 1 and Case 2, and the results of FedVeca algorithm and the baselines are shown in Fig. \ref{fig-case1&case2}. 

In Case 1, the curves of loss and accuracy values of FedVeca algorithm in each round overlap with the baseline methods, and both converge within $ K = 100 $ rounds (compared to Centralized SGD), proving that FedVeca algorithm is applicable on the IID dataset. In Case 2, FedVeca algorithm has a smaller difference in loss values and a larger difference in accuracy values in each round compared to the baseline methods, and both loss and accuracy values are the first to reach the convergence point. Together with the results of Case 3 we obtained in Section \ref{loss and acc}, we show that FedVeca algorithm is also applicable to the Non-IID dataset.

By comparing Fig. \ref{fig-case1&case2} and Fig. \ref{fig-loss-acc}, we note that on the SVM model and the MNIST dataset (convex loss function), the convergence rate of the model obtained by FedVeca algorithm is essentially the same in Case 1, Case 2 and Case 3. Meanwhile, the models obtained by FedNova and FedAvg perform better in Case1 and Case2 and worse in Case3. This proves that FedVeca algorithm has good performance and stability on both IID and Non-IID datasets. Therefore, we focus on Case 3 in the following and provide further insights on the prototype system.

\subsubsection{Instantaneous Behavior}
\label{instance}
\begin{figure*}[!t]
	\centering
	\includegraphics[width=6.5in]{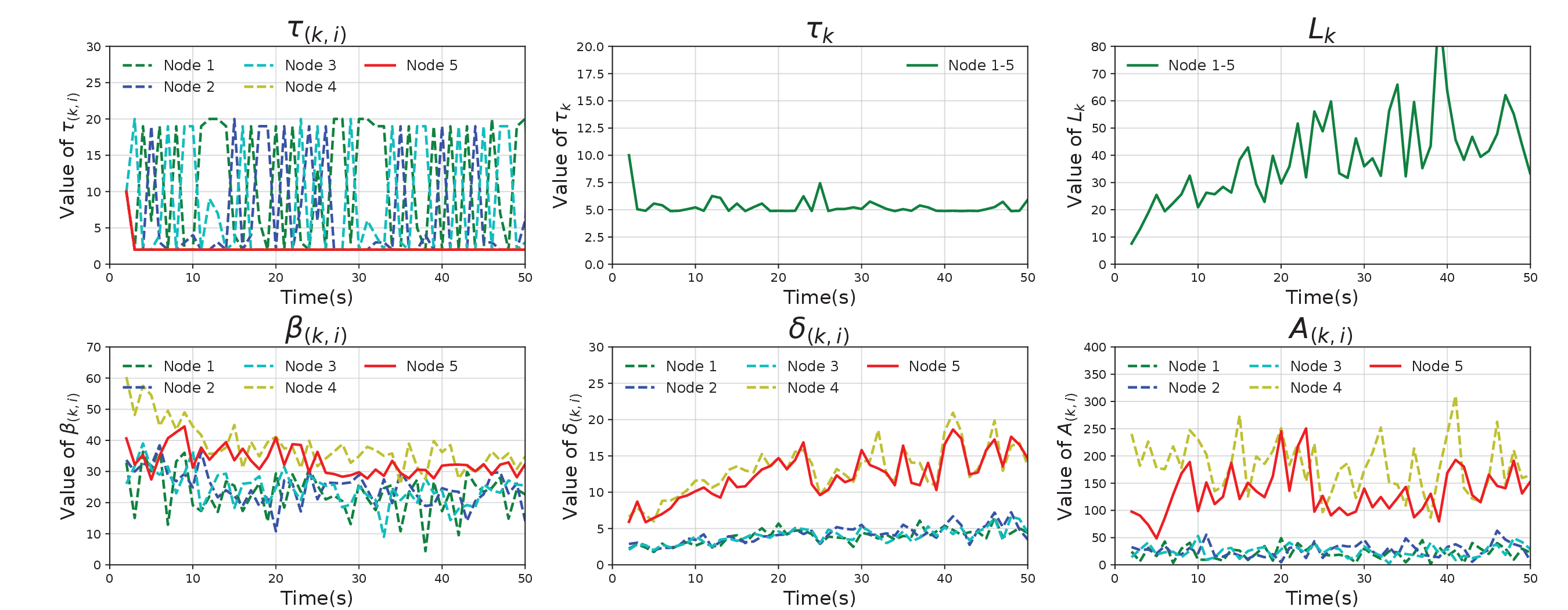}
	\caption{Instantaneous results of a single run (on the SVM model and the MNIST dataset \cite{lecun1998gradient}) with FedVeca algorithm in Case 3.}
	\label{fig-7}
\end{figure*}
We study the instantaneous behavior of FedVeca algorithm for a single run on the prototype system. Results for SVM (MNIST) are shown in Fig. \ref{fig-7}, where we record the amount of variation $ \tau_{(k, i)} $, $ \tau_{k} $, $ L_{k} $ , $ \beta_{(k,i)} $, $ \delta_{(k,i)} $ and $ A_{(k,i)} $ on the five clients, and represent only one curve for  $ \tau_{k} $ and $ L_{k} $. To make the image clearer, we only captured the first 50 rounds of the show.

We can see that the value of $ \tau_{(k, i)} $ fluctuates very much with round $ k $, and the maximum number of SGD iterations in each round is uniformly distributed across clients, indicating that FedVeca algorithm can adaptively choose the optimal $ \tau_{(k, i)} $ value according to the working conditions during the training process. $ \tau_{k} $ represents the step size of global gradient descent in Section \ref{fednova-alg}, and there is no large fluctuation in the curve of $ \tau_{k} $, indicating that the model converges smoothly globally despite the large difference in $ \tau_{(k, i)} $ values at each client under adaptive control. For the value of $ L_{k} $ and Assumption \ref{assumption-L}, we know that as the number of rounds increases, the variation between the model parameters of adjacent rounds becomes smaller, which represents the convergence of the model.

According to Theorem \ref{theorem-2}, $ A_{(k,i)} $ is the key variable for adaptive control of $ \tau_{(k,i)} $ values and $ A_{(k,i)} $ is composed of two variables, $ \beta_{(k,i)} $ and $ \delta_{(k,i)} $. In Fig. \ref{fig-7}, we can see that the values of $ A_{(k,i)} $ on Node 4 and Node 5 are widely spaced compared to those on the remaining three clients. This is related to the Case 3 we mentioned in Section \ref{sec-distribute}. The data distribution on Node 4 and Node 5 are similar but differ significantly from the remaining three clients, causing differences between the $ \beta_{(k,i)} $ and $ \delta_{(k,i)} $ values. Therefore, FedVeca algorithm accurately quantifies these differences and adaptively optimizes the FL training process by controlling the $ \tau_{(k, i)} $ values. 

\subsubsection{Sensitivity of $ \alpha_{k} $}
\label{alpha}
\begin{figure}[!t]
	\centering
	\includegraphics[width=3.5in]{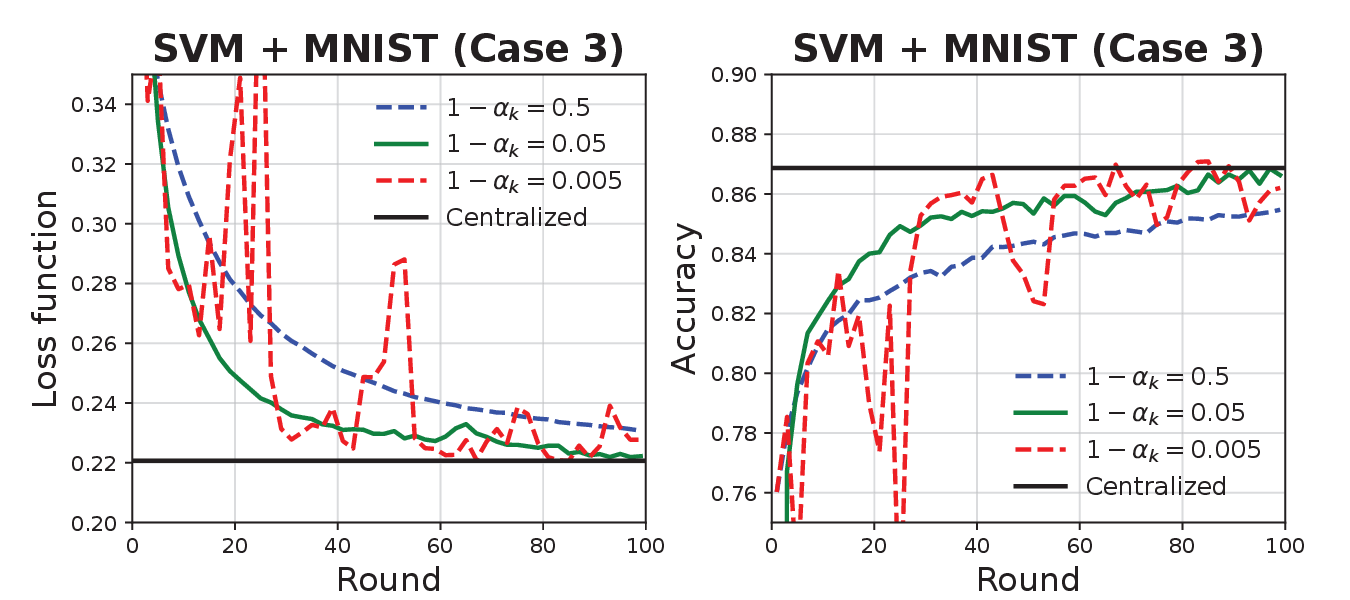}
	\caption{Loss function and classification accuracy (on the SVM model and the MNIST dataset \cite{lecun1998gradient}) with the different values of $ \alpha_{k} $.}
	\label{fig-alpha}
\end{figure}
In the setup of Section \ref{control}, the value of $ \alpha_{k} $ is fixed for each round and we know that the maximum value of $ \tau_{(k,i)} $ on all nodes in a round is related to $ 1-\alpha_{k} $ according to \eqref{eq-theorem-2}. Therefore, we choose three numbers 0.5, 0.05 and 0.005 for the values of $ 1-\alpha_{k} $ and record the variation of the model loss and accuracy values in these three scenarios, the results are shown in Fig. \ref{fig-alpha}. We can see that when $ 1-\alpha_{k} = 0.5 $, the loss and accuracy curves of the model are smooth, but their convergence rate is slow. When $ 1-\alpha_{k} = 0.005 $, the loss and accuracy values of the model reach the convergence point first, but their curves are not smooth. Therefore, in our previous experiments, we chose $ 1-\alpha_{k} = 0.05 $ which means $ \alpha_{k}=0.95 $. In this setting, the loss and accuracy values of the model have both a fast convergence rate and a smooth curve.

\subsubsection{Varying Number of clients}
\label{clients}
\begin{figure}[!t]
	\centering
	\includegraphics[width=3.5in]{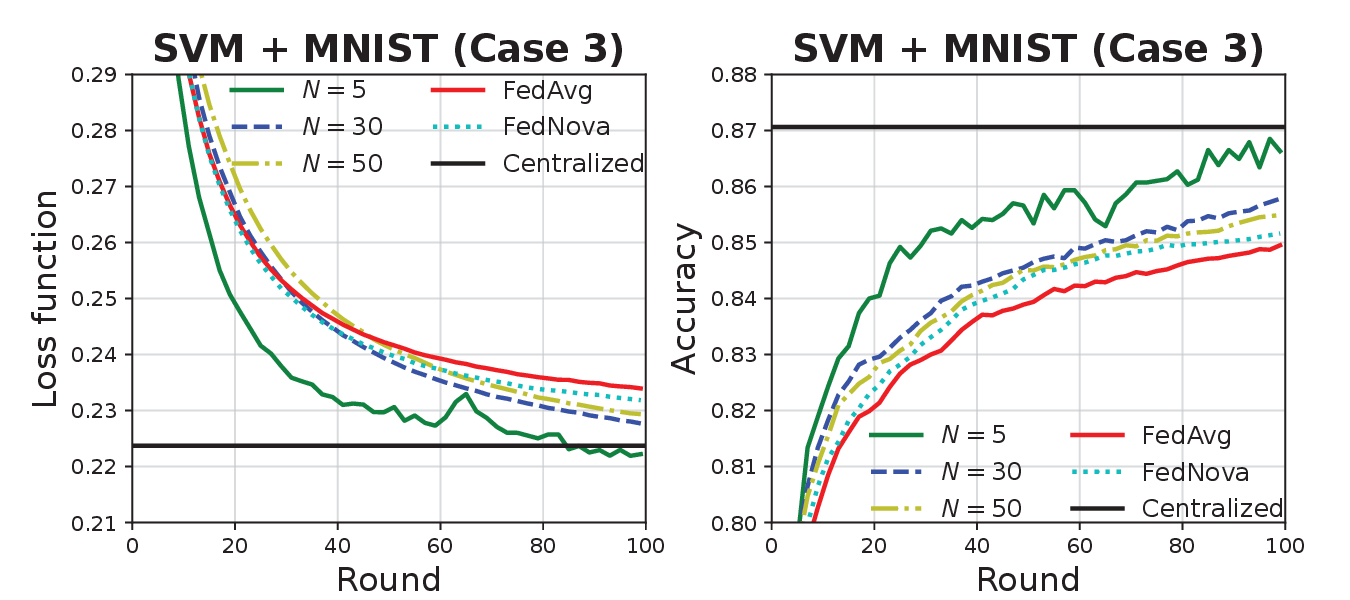}
	\caption{Loss function and classification accuracy (on the SVM model and the MNIST dataset \cite{lecun1998gradient}) with the different number of clients.}
	\label{fig-nodes}
\end{figure}
To simulate the case of multiple nodes, we let each Raspberry Pi 4B in the prototype system run multiple programs of Algorithm \ref{alg-node} in parallel and set the number of programs running on each Raspberry Pi 4B to be equal from 1 to 10, thus simulating the number of nodes from 5 to 50. We record the loss and accuracy values in Case 3, where the number of clients is chosen to be 5, 30 and 50 for FedVeca algorithm, FedNova and FedAvg selected the number of clients as 50, and the results are shown in Fig. \ref{fig-nodes}.

We can observe that the loss and accuracy values of the model do not converge faster as the number of clients increases, which is consistent with the phenomenon of diminishing returns which is described in \cite{mcmahan2017communication}. Also, since the size of our total training dataset is fixed, when the number of nodes increases, the number of samples on each node decreases and the data becomes more dispersed, thus reducing the speed of model convergence. However, even at 50 nodes, the model of FedVeca algorithm converges faster than FedAvg and FedNova approach, proving that FedVeca algorithm is applicable at multiple nodes.

Considering the overall experiments, we can conclude that FedVeca algorithm provides an efficient performance for FL, and provides a novel innovation with theory for FL optimization.

\section{Conclusion and The Future Work}
\label{sec-conclusion}
In order to improve the training efficiency of Federated Learning on the Non-IID dataset in communication networks, this paper proposed an method that is based on FedNova algorithm and controls the number of local SGD iterations on clients to obtain the optimal global training models in each communication round. We analyze the mathematical relationships between the number of local SGD iterations and the global objective in a round, and design an adaptive control algorithm from this relationship to predict the number of local SGD iterations on each client for the next round. Our experimental results confirmed the effectiveness of FedVeca method. In the future, we will explore how FedVeca method performs on the models with no-convex loss functions, and assign more effective weights for updating the global model based on the number and feature of samples per client. Furthermore, we intend to expand the application of the FedVeca algorithm to the rapidly growing field of Computer Vision (CV), specifically targeting more complex datasets such as CIFAR100 and ImageNet.

\section*{Acknowledgments}
This work was supported by the Key Research and Development Program of Hainan Province (Grant No. ZDYF2024GXJS014, ZDYF2023GXJS163), National Natural Science Foundation of China (NSFC) (Grant No. 62162022, 62162024), Collaborative Innovation Project of Hainan University(XTCX2022XXB02).



\normalem
\bibliographystyle{IEEEtran}
\bibliography{IEEEabrv,myrefs}
\vspace{11pt}
\begin{IEEEbiography}[{\includegraphics[width=1in,height=1.25in,clip,keepaspectratio]{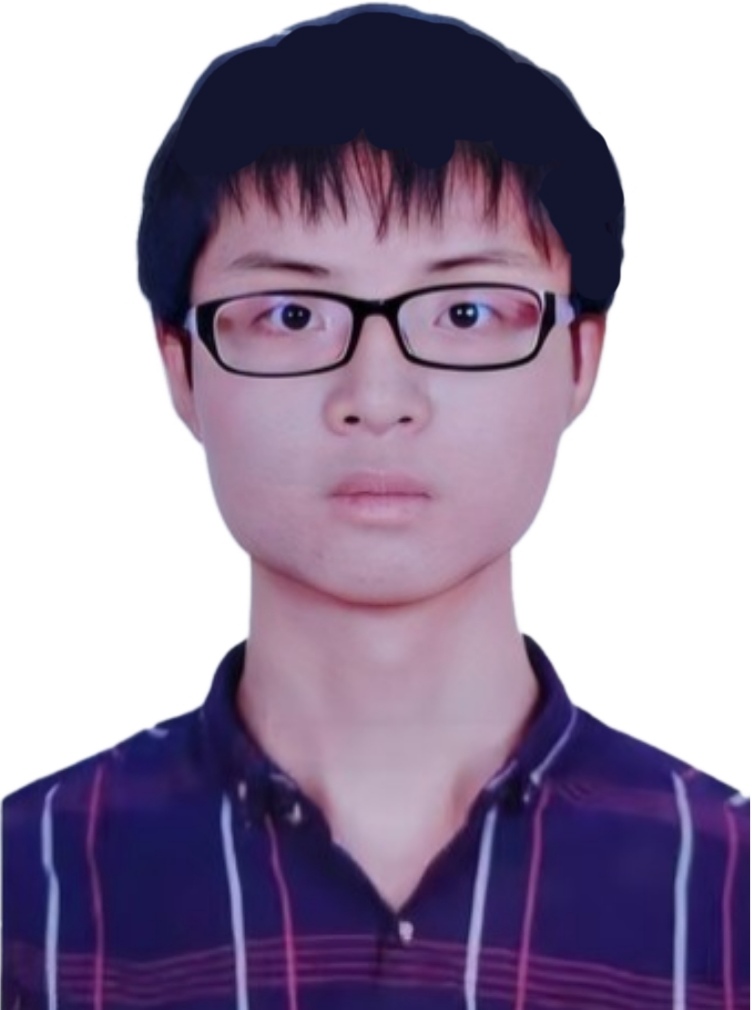}}]{Ping Luo}
	(Student Member, IEEE, 2022) is currently pursuing the Ph.D. degree in Computer Science and Technology from National University of Defense Technology (NUDT), Changsha, China. He received his MA.Eng. degree in Computer Science and Technology with Hainan University, Haikou, China. His research interests include Convex Optimization, the Industrial Internet of Things and Artificial Intelligence.
\end{IEEEbiography}
\vspace{11pt}

\begin{IEEEbiography}[{\includegraphics[width=1in,height=1.25in,clip,keepaspectratio]{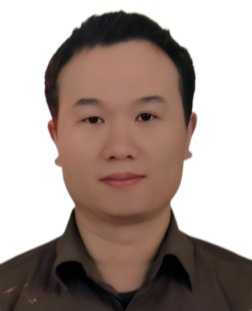}}]{Jieren Cheng}
	(Member, IEEE, 2020) is now a Professor and the Associate Dean of School of Computer Science \& Technology in Hainan University, China. He received his Ph.D. degree in Computer Science and Technology from National University of Defense Technology (NUDT) in 2010. He is awarded as ``Famous South China Sea Scholar". He serves as the director of the Hainan Provincial Blockchain Technology Engineering Research Center. His research interests include Blockchain, Big Data, Cloud Computing, Cyberecurity, Artificial Intelligence and Intelligent Transportation.
\end{IEEEbiography}

\vspace{11pt}
\begin{IEEEbiography}[{\includegraphics[width=1in,height=1.25in,clip,keepaspectratio]{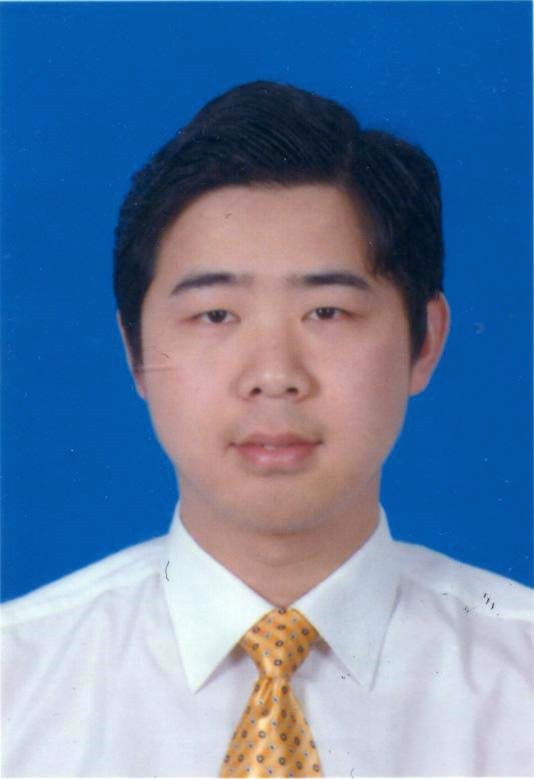}}]{N. Xiong}
	(S’05–M’08–SM’12) is current a Distinguished Professor at National Engineering Research Center for E-Learning, Central China Normal University (CCNU), Wuhan, Hu Bei Province, 430079, China. He is also with the Department of Computer Science, Georgia State University, Atlanta, GA 30302, USA. He received his PhD degree in School of Information Science, Japan Advanced Institute of Science and Technology (JAIST) on March 1, 2008. His research interests include Deep Learning, Reliable Networks, Software Engineering, and Big Data Analytics. 
	
	Dr. Xiong works in CCNU for many years, and obtained many research funding and many industrial projects. He published over 600 journal paper with 200+ IEEE journal papers. He also creates a company about design and analysis for complex reliable software systems, and obtains over 10 patents.
\end{IEEEbiography}

\vspace{11pt}
\begin{IEEEbiography}[{\includegraphics[width=1in,height=1.25in,clip,keepaspectratio]{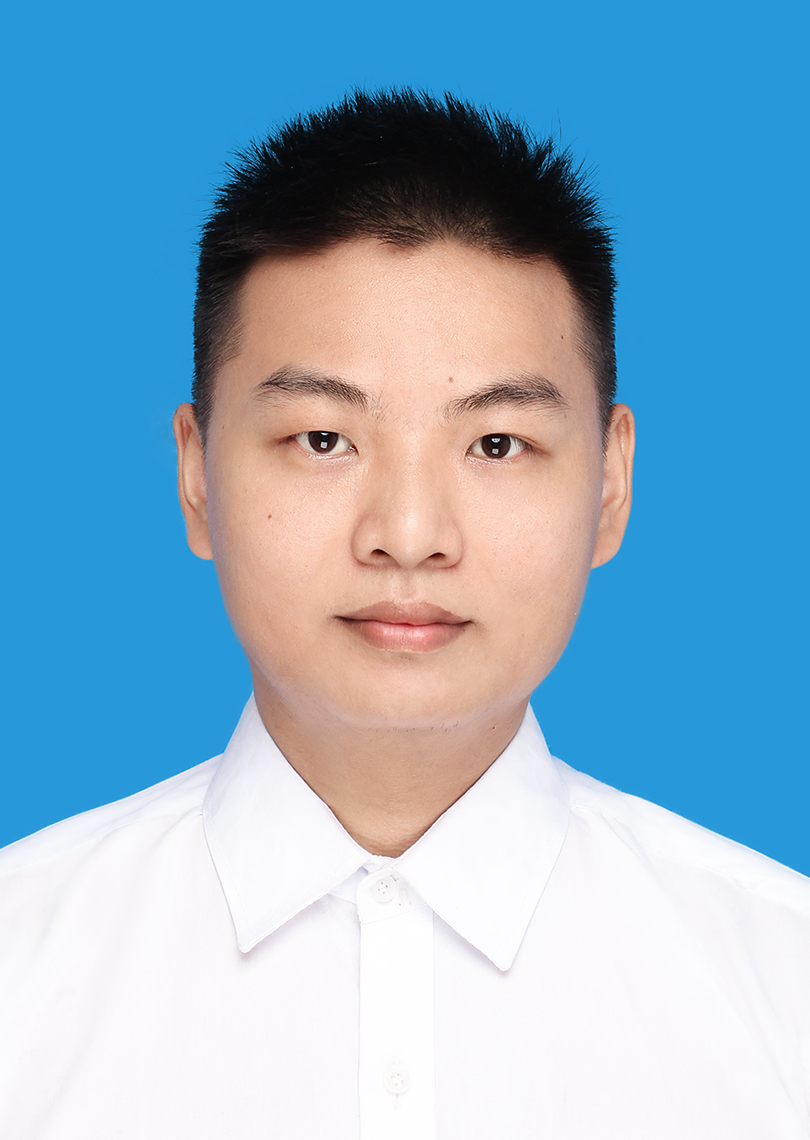}}]{Zhenhao Liu}
	(Student Member, IEEE, 2022) is currently working toward the Ph.D. degree in Agricultural Information Engineering with Huazhong Agricultural University, Wuhan, China. He received his MA.Eng. degree in Electronic Information with Hainan University, Haikou, China. His research interests include Federated Learning, Privacy-preserving, and Artificial Intelligence.
\end{IEEEbiography}
\vspace{11pt}

\begin{IEEEbiography}[{\includegraphics[width=1in,height=1.25in,clip,keepaspectratio]{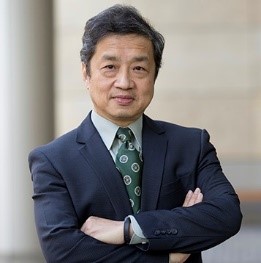}}]{Jie Wu}
	(Fellow, IEEE, 2009) is the Director of the Center for Networked Computing and Laura H. Carnell professor at Temple University. He also serves as the Director of International Affairs at College of Science and Technology. He served as Chair of Department of Computer and Information Sciences from the summer of 2009 to the summer of 2016 and As- sociate Vice Provost for International Affairs from the fall of 2015 to the summer of 2017. Prior to joining Temple University, he was a program director at the National Science Foundation and was a distinguished professor at Florida Atlantic University. His current research interests include mobile computing and wireless networks, rout- ing protocols, network trust and security, distributed algorithms, applied machine learning, and cloud computing. Dr. Wu regularly publishes in scholarly journals, conference proceedings, and books. He serves on several editorial boards, including IEEE Transactions on Mobile Computing, IEEE Transactions on Service Computing, Journal of Parallel and Distributed Computing, and Journal of Computer Science and Technology. Dr. Wu is/was general chair/co-chair for IEEE IPDPS’08, IEEE DCOSS’09, IEEE ICDCS’13, ACM MobiHoc’14, ICPP’16, IEEE CNS’16, WiOpt’21, and ICDCN’22 as well as program chair/cochair for IEEE MASS’04, IEEE INFO- COM’11, CCF CNCC’13, and ICCCN’20. He was an IEEE Computer Society Dis- tinguished Visitor, ACM Distinguished Speaker, and chair for the IEEE Technical Committee on Distributed Processing (TCDP). Dr. Wu is a Fellow of the AAAS and a Fellow of the IEEE. He is the recipient of the 2011 China Computer Federation (CCF) Overseas Outstanding Achievement Award.
\end{IEEEbiography}

\vspace{11pt}

\vfill
\clearpage
{\appendices
\section{Proof of Theorem \ref{theorem-1}}
\label{proof-theo1}
\subsection{Preliminaries}
	Because each global gradient is Lipschitz smooth according to the Assumption \ref{assumption-L}, and we have
\begin{equation}
	f(\boldsymbol{x})-f(\boldsymbol{y})-\nabla f(\boldsymbol{y})^{T}(\boldsymbol{x}-\boldsymbol{y}) \leq \frac{L}{2}\|\boldsymbol{x}-\boldsymbol{y}\|^{2}
\end{equation}
for arbitrary $ \boldsymbol{x} $ and $ \boldsymbol{y} $ \cite{bubeck2015convex}, where $ \nabla f(\boldsymbol{.})^{T} $ is the transpose of $ \nabla f(\boldsymbol{.}) $. Thus, for $ k \in [0,k-1] $, we have
\begin{equation}
	\label{eq-lipschitz}
	\begin{aligned}
		&F\left(\boldsymbol{w}_{k+1}\right)-F\left(\boldsymbol{w}_{k}\right) \\
		\leq&-\nabla F\left(\boldsymbol{w}_{k}\right)\left(\boldsymbol{w}_{k+1}-\boldsymbol{w}_{k}\right)+\frac{L}{2}\left\|\boldsymbol{w}_{k+1}-\boldsymbol{w}_{k}\right\|^{2}.
	\end{aligned}
\end{equation}
Substituting \eqref{eq-fednova} into \eqref{eq-lipschitz} yields
\begin{equation}
	\label{eq-unfold}
	\begin{aligned}
		&F\left(\boldsymbol{w}_{k+1}\right)-F\left(\boldsymbol{w}_{k}\right) \\
		\leq&-\eta \tau_{k}\left\langle\nabla F\left(\boldsymbol{w}_{k}\right), \boldsymbol{d}_{k}\right\rangle +\frac{\eta^{2} \tau_{k}^{2} L}{2}\left\|\boldsymbol{d}_{k}\right\|^{2}.
	\end{aligned}
\end{equation}
We solve for the inner product term of \eqref{eq-unfold} to obtain
\begin{equation}
	\label{eq-inner}
	\begin{aligned}
		&\left\langle \nabla F\left(\boldsymbol{w}_{k}\right), \boldsymbol{d}_{k}\right\rangle \\
		=&\frac{1}{2}\left\|\nabla F\left(\boldsymbol{w}_{k}\right)\right\|^{2}+\frac{1}{2}\left\|\boldsymbol{d}_{k}\right\|^{2} -\frac{1}{2} T_{k},
	\end{aligned}
\end{equation}
where $ T_{k} \triangleq \left\|\nabla F\left(\boldsymbol{w}_{k}\right)-\boldsymbol{d}_{k}\right\|^{2} $ and the equation uses the fact $ 2 \left\langle a, b \right\rangle = \left\| a \right\|^{2} + \left\| b \right\|^{2} - \left\| a -b \right\|^{2} $. Substituting \eqref{eq-inner} into \eqref{eq-unfold} yields
\begin{equation}
	\label{F-ini}
	\begin{aligned}
		&\frac{F\left(\boldsymbol{w}_{k+1}\right)-F\left(\boldsymbol{w}_{k}\right)}{\eta \tau_{k}} \\
		\leq &-\frac{1}{2}\left(\left\|\nabla F\left(\boldsymbol{w}_{k}\right)\right\|^{2}+\left\|\boldsymbol{d}_{k}\right\|^{2}- T_{k} \right) +\frac{\eta \tau_{k} L}{2}\left\|\boldsymbol{d}_{k}\right\|^{2} \\
		=&-\frac{1}{2}\left\|\nabla F\left(\boldsymbol{w}_{k}\right)\right\|^{2}+\frac{\left(\eta \tau_{k} L-1\right)}{2}\left\|\boldsymbol{d}_{k}\right\|^{2} + \frac{1}{2} T_{k}\\
		\leq &(\frac{\eta \tau_{k} L}{2} - 1)\left\|\nabla F\left(\boldsymbol{w}_{k}\right)\right\|^{2}+\frac{1}{2} T_{k}\\
		&\text{(when $ \eta \tau_{k} L-1 \geq 0 $)},
	\end{aligned}
\end{equation}
where the last inequality is because $ \left\|\boldsymbol{d}_{k}\right\|^{2} \leq \left\|\nabla F\left(\boldsymbol{w}_{k}\right)\right\|^{2} $ according to the Assumption \ref{assumption-d}. Next we will convert $ T_{k} $ to contain only the similar items of $ \left\|\nabla F\left(\boldsymbol{w}_{k}\right)\right\|^{2} $.

\subsection{Results for $ T_{k} $}

According to the definition of $ T_{k} $  and $ \boldsymbol{d}_{k} $ for $ i \in [1, N] $ and $ k \in [0, K-1] $, we always have

\begin{equation}
	\label{Tk}
	\begin{aligned}
		T_{k}&=\big\|\nabla F\left(\boldsymbol{w}_{k}\right)-\sum_{i=1}^{N} p_{i} \boldsymbol{G}_{(k, i)}\big\|^{2} \\
		&=\big\|\sum_{i=1}^{N} p_{i}\left(\nabla F_{i}\left(\boldsymbol{w}_{k}\right)-\boldsymbol{G}_{(k, i)}\right)\big\|^{2} \\
		&\text{(from the definition of $ \nabla F\left(\boldsymbol{w}_{k}\right) $ in \eqref{eq-global grad})}\\
		&\leq \sum_{i=1}^{N} p_{i}\left\|\nabla F_{i}\left(\boldsymbol{x}_{k}\right)-\boldsymbol{G}_{(k, i)}\right\|^{2},
	\end{aligned}
\end{equation}
where the last inequality uses Jensen’s Inequality: $ \big\| \sum_{i=1}^{N} \varphi_{i} \boldsymbol{z} \big\|^{2} \leq \sum_{i=1}^{N} \varphi_{i} \left\| \boldsymbol{z} \right\|^{2} $ for arbitrary $ \boldsymbol{z} $ and $ \sum_{i=1}^{N} \varphi_{i} = 1 $. Then, we solve for $ \left\|\nabla F_{i}\left(\boldsymbol{x}_{k}\right)-\boldsymbol{G}_{(k, i)}\right\|^{2} $ and get
\begin{equation}
	\label{assume-beta}
	\begin{aligned}
		&\left\|\nabla F_{i}\left(\boldsymbol{w}_{k}\right)-\boldsymbol{G}_{(k, i)}\right\|^{2}\\
		=&\big\|\nabla F_{i}\left(\boldsymbol{w}_{k}\right)-\frac{1}{\tau_{(k, i)}} \sum_{\lambda=0}^{\tau_{(k, i)}-1} \nabla F_{i}\big(\boldsymbol{w}_{(k, i)}^{\lambda}\big)\big\|^{2} \\
		=&\big\|\sum_{\lambda=0}^{\tau_{(k, i)}-1} \frac{1}{\tau_{(k, i)}}\Big(\nabla F_{i}\left(\boldsymbol{w}_{k}\right)-\nabla F_{i}\big(\boldsymbol{w}_{(k, i)}^{\lambda}\big)\Big)\big\|^{2} \\
		\leq &\sum_{\lambda=0}^{\tau_{(k, i)}-1} \frac{1}{\tau_{(k, i)}}\big\|\nabla F_{i}\left(\boldsymbol{w}_{k}\right)-\nabla F_{i}\big(\boldsymbol{w}_{(k, i)}^{\lambda}\big)\big\|^{2} \\
		&\text{(from Jensen’s Inequality)} \\
		\leq &\sum_{\lambda=0}^{\tau_{(k, i)}-1} \frac{\beta_{(k, i)}^{2}}{\tau_{(k, i)}}\big\|\boldsymbol{w}_{k}-\boldsymbol{w}_{(k, i)}^{\lambda}\big\|^{2}\\
		&\text{(from the Assumption \ref{assumption-beta})}
	\end{aligned}
\end{equation}
for $ \lambda \in [0, \tau_{(k, i)}-1] $. According to our definition in Section \ref{sec-preliminaries}, we have $ \boldsymbol{w}_{k} = \boldsymbol{w}_{(k, i)}^{\lambda} $ at $ \lambda = 0 $, at this point the last term of the inequation in \eqref{assume-beta} has $ \big\|\boldsymbol{w}_{k}-\boldsymbol{w}_{(k, i)}^{0}\big\|^{2} = 0 $. Therefore, we solve for $ \big\|\boldsymbol{w}_{k}-\boldsymbol{w}_{(k, i)}^{\lambda}\big\|^{2} $ with $ \lambda \in [1, \tau_{(k, i)}-1] $ and obtain
\begin{equation}
	\label{assume-delta}
	\begin{aligned}
		&\big\|\boldsymbol{w}_{k}-\boldsymbol{w}_{(k, i)}^{\lambda}\big\|^{2}=\eta^{2}\big\|\sum_{s=0}^{\lambda-1} \nabla F_{i}\big(\boldsymbol{w}_{(k, i)}^{s}\big)\big\|^{2} \\
		&\text{(from the SGD rule in \eqref{local update})}\\
		\leq&\eta^{2} \delta_{(k, i)} \sum_{s=0}^{\lambda-1}\left\|\nabla F\left(\boldsymbol{w}_{k}\right)\right\|^{2} \\
		&\text{(from the Assumption \ref{assumption-delta})}\\
		=&\eta^{2} \delta_{(k, i)} \lambda\left\|\nabla F\left(\boldsymbol{w}_{k}\right)\right\|^{2}.  
	\end{aligned}
\end{equation}
Substituting \eqref{assume-delta} into \eqref{assume-beta}, we get
\begin{equation}
	\label{defi-A}
	\begin{aligned}
		&\left\|\nabla F_{i}\left(\boldsymbol{w}_{k}\right)-\boldsymbol{G}_{(k, i)}\right\|^{2}\\
		\leq &\sum_{\lambda=0}^{\tau_{(k, i)}-1} \lambda \eta^{2} \frac{\beta_{(k, i)}^{2} \delta_{(k, i)}}{\tau_{(k, i)}}\left\|\nabla F\left(\boldsymbol{w}_{k}\right)\right\|^{2}\\
		=& \frac{\tau_{(k, i)}-1}{2} \eta^{2} \beta_{(k, i)}^{2} \delta_{(k, i)} \left\|\nabla F\left(\boldsymbol{w}_{k}\right)\right\|^{2} \\
		=& \frac{\tau_{(k, i)}-1}{2} \eta A_{(k,i)} \left\|\nabla F\left(\boldsymbol{w}_{k}\right)\right\|^{2}\\
		&\text{(from the definition that $ A_{(k,i)} \triangleq \eta \beta_{(k,i)}^{2} \delta_{(k,i)}) $.}
	\end{aligned}
\end{equation}
Substituting \eqref{defi-A} into \eqref{Tk}, we get
\begin{equation}
	\label{Tk-final}
	\begin{aligned}
	T_{k} &\leq \sum_{i=1}^{N} p_{i} \frac{\tau_{(k, i)}-1}{2} \eta A_{(k,i)} \left\|\nabla F\left(\boldsymbol{w}_{k}\right)\right\|^{2}\\
	&=\frac{1}{2} \eta \left\|\nabla F\left(\boldsymbol{w}_{k}\right)\right\|^{2} \sum_{i=1}^{N} p_{i} (\tau_{(k, i)}-1) A_{(k,i)}.
	\end{aligned}
\end{equation}
Now, we are ready to derive the final result.
\subsection{Final Result}
Plugging \eqref{Tk-final} back into \eqref{F-ini}, we have
\begin{equation}
	\label{F-fianl}
	\begin{aligned}
		&\frac{F\left(\boldsymbol{w}_{k+1}\right)-F\left(\boldsymbol{w}_{k}\right)}{\eta \tau_{k}} \\
		\leq &(\frac{\eta \tau_{k} L}{2} - 1)\left\|\nabla F\left(\boldsymbol{w}_{k}\right)\right\|^{2}+ \\
		&\frac{1}{4} \eta \left\|\nabla F\left(\boldsymbol{w}_{k}\right)\right\|^{2}\sum_{i=1}^{N} p_{i} (\tau_{(k, i)}-1) A_{(k,i)}\\
		= & \left\|\nabla F\left(\boldsymbol{w}_{k}\right)\right\|^{2}\big( -1 + \frac{1}{4} \eta (2 \tau_{k} L + \sum_{i=1}^{N} p_{i} (\tau_{(k, i)}-1) A_{(k,i)})\big) \\
		= & \big( -1 + \frac{1}{4} \eta (\sum_{i=1}^{N} p_{i}\tau_{(k, i)} 2 L + \sum_{i=1}^{N} p_{i} (\tau_{(k, i)}-1) A_{(k,i)})\big) \\
		&\left\|\nabla F\left(\boldsymbol{w}_{k}\right)\right\|^{2}\\
		&\text{(from the defition of $ \tau_{k} $ in \eqref{eq-fednova})}\\
		= & \big(\frac{1}{4} \eta \sum_{i=1}^{N} p_{i}\left(\tau_{(k,i)}\left(2 L+A_{(k,i)} \right)-A_{(k,i)}\right) -1 \big) \left\|\nabla F\left(\boldsymbol{w}_{k}\right)\right\|^{2}.
	\end{aligned}
\end{equation}
Since both $\eta \tau_{k}$ and $\left\|\nabla F\left(\boldsymbol{w}_{k}\right)\right\|^{2}$ are larger than 0, we can rearrange \eqref{F-fianl} to get
\begin{equation}
    \begin{aligned}
			&\frac{F\left(\boldsymbol{w}_{k+1}\right)-F\left(\boldsymbol{w}_{k}\right)}{\left\|\nabla F\left(\boldsymbol{w}_{k}\right)\right\|^{2}} \\
			\leq &\eta \tau_{k}\big(\frac{1}{4} \eta \sum_{i=1}^{N} p_{i}\left(\tau_{(k,i)}\left(2 L+A_{(k,i)} \right)-A_{(k,i)}\right) -1 \big)
\end{aligned} 
\end{equation}

\section{Proof of Theorem}
\label{proof-theo2}
From Theorem \ref{theorem-1} we get that $ \eta \tau_{k} L -1 \geq 0 $ which means $ \eta \sum_{i=1}^{N} p_{i} \tau_{(k, i)} L \geq 1 $ in the $ k $-th round ($ k \in [0,k-1] $). Thus, it exists a real number $\varepsilon_{k}$ such that
\begin{equation}
	\frac{1}{4}\left( \varepsilon_{k} + 2 \right) \eta \sum_{i=1}^{N} p_{i} \tau_{(k, i)} L < 1.
\end{equation}
Obviously, $ \varepsilon_{k} < 2 $, because when $ \varepsilon_{k} \geq 2 $, we have a self-contradictory inequality $ \frac{1}{4}\left( \varepsilon_{k} + 2 \right) \eta \sum_{i=1}^{N} p_{i} \tau_{(k, i)} L \geq \eta \sum_{i=1}^{N} p_{i} \tau_{(k, i)} L \geq 1 $. And to ensure the model convergence in Theorem \ref{theorem-1} we have
\begin{equation}
	\label{eq-eps}
	\begin{aligned}
	&\frac{1}{4} \eta \sum_{i=1}^{N} p_{i}\left(\tau_{(k,i)}\left(2 L+A_{(k,i)} \right)-A_{(k,i)}\right) \\
	\leq &\frac{1}{4}\left( \varepsilon_{k} + 2 \right) \eta \sum_{i=1}^{N} p_{i} \tau_{(k, i)} L\\
	= & \frac{1}{4} \eta \sum_{i=1}^{N} p_{i} \tau_{(k, i)} L \left( \varepsilon_{k} + 2 \right).
	\end{aligned}
\end{equation}
Since the $L$-value and $\epsilon_{k}$-value are fixed at the current k-th round, the inequality \eqref{eq-eps} can be used by comparing $\tau_{(k,i)}\left(2 L+A_{(k, i)}\right)-A_{(k, i)}$ and $\tau_{(k,i)} L \left(\epsilon_{k}+2\right)$ on each client, that is
\begin{equation}
	\label{up-eps}
	\begin{aligned}
		\tau_{(k,i)}\left(2 L+A_{(k, i)}\right)-A_{(k, i)} &\leq \tau_{(k,i)} L \left(\epsilon_{k}+2\right)\\
		\tau_{(k,i)}\left(A_{(k, i)}-\epsilon_{k} L\right) &\leq A_{(k, i)} \\
		 \tau_{(k,i)} &\leq \frac{A_{(k, i)}}{A_{(k, i)}-\epsilon_{k} L}\\
        &\text{(when $ A_{(k, i)}-\epsilon_{k} L > 0 $)}.
	\end{aligned}
\end{equation}
Since we set the lower bound of $ \tau_{(k,i)} > 1 $, we have $  \frac{A_{(k, i)}}{A_{(k, i)}-\epsilon_{k} L} > 1 $ which is equivalent to $ \epsilon_{k} > 0 $, and get $ \varepsilon_{k} \in (0,2) $. Then, we define the existence of another real number $ \alpha_{k} $ that
\begin{equation}
	\label{eps2alp}
	\epsilon_{k} L=\alpha_{k} \min _{i} A_{(k, i)}.
\end{equation}
We know that $ \alpha_{k} = \frac{\epsilon_{k} L}{\min _{i} A_{(k, i)}} $, thus $ \alpha_{k} \in (0, \frac{2 L}{\min _{i} A_{(k, i)}}) $. To ensure that $ A_{(k, i)}-\epsilon_{k} L > 0  $ is established, we impose further constraints on $ \epsilon_{k} $ that $ \alpha_{k} \in (0, \frac{2L}{\min_{i} A_{(k,i)}}) $ (when $ \frac{2L}{\min_{i} A_{(k,i)}} < 1 $) and $ \alpha_{k} \in (0,1) $ (when $ \frac{2L}{\min_{i} A_{(k,i)}} > 1 $). Substituting \eqref{up-eps} into \eqref{eps2alp}, we get
\begin{equation}
	 \tau_{(k,i)} \leq \frac{A_{(k, i)}}{A_{(k, i)}-\alpha_{k} \min _{i} A_{(k, i)}}.
\end{equation}
}
\end{document}